\documentclass[11pt]{article}
\usepackage{fullpage}

\date{}

\title{\bfseries\papertitle}

\author{ Suprovat Ghoshal$^*$ \\ University of Michigan \\ suprovat@umich.edu	\and Aadirupa Saha$^*$ \\ Microsoft Research, New York City \\ aadirupa.saha@microsoft.com}

\usepackage[round]{natbib}
\usepackage[T1]{fontenc}
\usepackage{lmodern}
\usepackage[utf8]{inputenc} 
\usepackage{nameref}
\usepackage[colorlinks, linkcolor=blue, filecolor = blue, citecolor = blue, urlcolor = blue]{hyperref}
\usepackage{url}            
\usepackage{booktabs}       
\usepackage{amsfonts}       
\usepackage{nicefrac}       
\usepackage{microtype}      
\usepackage{xcolor}         

\usepackage{microtype}
\usepackage{cancel}
\usepackage{graphicx}
\usepackage{booktabs} 
\usepackage{amssymb}
\usepackage{color}
\usepackage{amsmath}
\usepackage{amsthm}
\usepackage{thmtools}
\usepackage{thm-restate}
\usepackage{algorithm}
\usepackage{algorithmic}
\usepackage{dsfont}

\usepackage{natbib}
\usepackage{hyperref}
\usepackage{nameref}

\newtheorem{thm}{Theorem}
\newtheorem{lem}[thm]{Lemma}

\newtheorem{cor}[thm]{Corollary}
\newtheorem{defn}[thm]{Definition}

\newtheorem{rem}{Remark}

\newtheorem{cl}{Claim}

\newcommand{\R}{{\mathbb R}}

\newcommand{\E}{{\mathbf E}}

\newcommand{\1}{{\mathbf 1}}
\newcommand{\0}{{\mathbf 0}}
\newcommand{\cA}{{\mathcal A}}

\newcommand{\cB}{{\mathcal B}}

\newcommand{\cI}{{\mathcal I}}

\newcommand{\cS}{{\mathcal S}}
\newcommand{\cE}{{\mathcal E}}

\newcommand{\cF}{{\mathcal F}}
\newcommand{\cG}{{\mathcal G}}
\newcommand{\cT}{{\mathcal T}}

\newcommand{\cN}{{\mathcal N}}

\newcommand{\cD}{{\mathcal D}}

\newcommand{\tS}{{\tilde {S}}}
\newcommand{\tj}{{\tilde {j}}}

\newcommand{\teta}{{\tilde \eta}}

\newcommand{\p}{{\mathbf p}}

\newcommand{\sm}{\setminus}

\newcommand{\bzeta}{\boldsymbol \zeta}

\newcommand{\bSigma}{\boldsymbol \Sigma}

\newcommand{\bnu}{{\boldsymbol \nu}}
\newcommand{\bmu}{{\boldsymbol \mu}}

\newcommand{\bsigma}{\boldsymbol \sigma}

\DeclareMathOperator*{\argmin}{argmin}
\DeclareMathOperator*{\argmax}{argmax}

\def \gn{{\it Gumbel}}

\def \rank{{\texttt{Rank}}}
\newcommand{\obr}{\textit{Ordered Block-Rank}}

\newcommand{\br}{\textit{Block-Rank}}

\newcommand{\algbr}{\texttt{Block-Rank Preference Bandits}}
\newcommand{\brpb}{BlockRank-PB}

\newcommand{\rumr}{\texttt{LR-RUM}$(n,k,r)$}
\newcommand{\rumb}{\texttt{BR-RUM}$(n,k,r)$}
\newcommand{\irum}{\texttt{I-RUM}$(n,k)$}
\def \ratio{{\it Best-Item-Advantage-Ratio}}
\def \mrat{{\it $\epsilon$-BAR}}
\def \lr{{Low-Rank-Choice-Model}}
\def \indrum{{Independent-RUM-Choice-Model}}
\def \prob{{$(\epsilon,\delta)$-PAC arm identification in LR-RUM}}

\def \papertitle{{Exploiting Correlation to Achieve Faster Learning Rates  in Low-Rank Preference Bandits}} 

\newcommand{\red}[1]{\textcolor{red}{#1}}

\let\Pr\relax

\DeclareMathOperator*{\Pr}{Pr}

\newcommand{\oleq}[1]{\overset{#1}{\leq}}

\newcommand{\cX}{\mathcal{X}}

\begin{document}

\maketitle

\def\thefootnote{*}\footnotetext{Equal contribution alphabetically.}

\vspace{-10pt}
\begin{abstract}
We introduce the \emph{Correlated Preference Bandits} problem with random utility based choice models (RUMs), where the goal is to identify the best item from a given pool of $n$ items through online subsetwise preference feedback. 
We investigate whether models with a simple correlation structure, e.g. low rank, can result in faster learning rates. 
While we show that the problem can be impossible to solve for the general `low rank' choice models, faster learning rates can be attained assuming more structured item correlations. In particular, we introduce a new class of \emph{Block-Rank} based RUM model, where the best item is shown to be $(\epsilon,\delta)$-PAC learnable with only $O(r \epsilon^{-2} \log(n/\delta))$ samples. This improves on the standard sample complexity bound of $\tilde{O}(n\epsilon^{-2} \log(1/\delta))$ known for the usual learning algorithms which might not exploit the item-correlations ($r \ll n$). 
We complement the above sample complexity with a matching lower bound (up to logarithmic factors), justifying the tightness of our analysis. 
Further, we extend the results to a more general `\emph{noisy Block-Rank}' model, which ensures robustness of our techniques. Overall, our results justify the advantage of playing subsetwise queries over pairwise preferences $(k=2)$, we show the latter provably fails to exploit correlation. 
\end{abstract}

\vspace{-15pt}
\section{Introduction}
\label{sec:intro}
\vspace{-7pt}
We give an algorithm for sequentially PAC learning the best item from a finite pool of $n$ items, where at each decision round $t$, a subset of $k$ items can be tested, and preference feedback of the winning item can be observed. Given a fixed $\epsilon,\delta \in [0,1]$, the objective of the algorithm is to find, with high probability $1-\delta$, an `$\epsilon$-best' item, with minimum possible query complexity. 

The problem has been studied extensively in recent works in the setting of pairwise preferences (i.e. $k = 2$) \citep{Busa_pl,falahatgar_nips,Busa14survey}, while some works also extend the setting to general subsetwise queries \citep{Ren+18,SGwin18,ChenSoda+18} for specific choice of random utility model (RUM) based subset choice model (e.g.  MNL or Plackett-Luce model \cite{Agrawal+16}). While at a first glance, one might expect the sample complexity of an optimal learner to depend on the sizes of the subsets queried (i.e., $k$)---precisely, with increasing subset size $k$, one may expect to achieve faster learning rates, as with larger $k$, the learner also gets to observe a preference feedback on  more items per time step. 

However, surprisingly, it is known that in general, the fundamental performance limit of the problem is not improvable based on the subset size $k$. For e.g., \cite{SG18,Ren+18} formally shows a worst-case sample complexity lower bound of $O(\frac{n}{\epsilon^2}\log \frac{1}{\delta})$ for any $k \in [n]$ which has no dependence on $k$. These results are of course discouraging, since they imply there is no advantage in observing general subsetwise feedback over pairwise preferences ($k=2$). Why should one even build systems for general $k$-subset size queries when a pairwise query serves as good? 

Our first step towards answering the above is the following crucial observation: the subset size obliviousness of the earlier results is rooted in the fact that here aforementioned results assume a \emph{no-correlations among the item rewards} structure. But this in turn implies that the winning probability of a certain item in a $k$-subset solely depends on its own underlying value, and is independent of the context (rest of items present alongside), which is often unjustified in practice. In almost every real-world scenario, the items in the decision space are often correlated with interdependent utilities or losses; e.g. in movie recommendation, if a group of users dislikes a movie from the horror genre, it is likely they will dislike a thriller movie as well, similarly in restaurant recommendation, if a person shows preference for `tiramisu', one may expect the similar desert items are also going to be in the top of their preference list etc. Thus depending on the nature of correlations, correlations may help in faster information aggregation where the learner can hope to gather side information about related items without explicitly learning the underlying scores (rewards) of each item separately.

Assuming `independent rewards' however defeats the purpose of subsetwise games. This is since despite having the provision of playing a larger set of items (and hence observing feedback on a larger item set per round), due to the `independence' assumption, the preference outcome of one item does not reveal any information of the rest as their scores remains unaffected by each other's presence. We thus focus our attention to studying the interplay between learning rate and reward correlations in preference bandits: Here the preference information of one item can reveal additional partial preferential information of the items present alongside and hence, one can hope that selecting larger subsets in such settings should lead to faster learning rates (smaller sample complexity).

As mentioned above, to the best of our knowledge, none of the earlier work address this perspective in the setting of preference bandits, arguably due to the ease of analyzing their proposed algorithms under the `independent (uncorrelated)' assumption, e.g., \cite{SGwin18,KhetanOh16,ChenSoda+18} exploit the {Independence of Irrelevant Attributes} (IIA) property of the Plackett Luce (PL) preference model in their sample complexity analysis. In fact, it is unclear how to incorporate `correlation structures' into subsetwise preference models. 

The main objective of our work is to \emph{formulate and understand how playing a subsetwise game can improve the sample complexity of the best arm identification problem for correlated items} (\emph{without the learner having prior knowledge of the underlying correlation structure}). Our contributions are: 

(1) We introduce the problem of \emph{Correlated Preference Bandits} under random utility based preference models (RUMs)\footnote{It is worth noting that correlated noise in discrete choice models has been studied in statistics and economics (e.g.,~\cite{train09}), however it is not known how to exploit the correlation structure to achieve faster learning rates through preference based active learning, which remains the goal of this work.}, which generalizes the \indrum\, model by incorporating item correlations in terms of \lr s  \rumr \, (see Sec. \ref{sec:prelims} and \ref{sec:prob}).
 
(2). Our first finding shows that for any general \lr \, \rumr, the best-arm identification problem can be impossible to solve for (see Lem. \ref{thm:lb_gen}, Sec. \ref{sec:genrank}). 
 
(3). We then introduce a new class of \br\, based RUM model which uses a more combinatorially interpretable notion of rank. We show that in the setting of RUMs with block rank at most $r$, namely \rumb\, the best item is $(\epsilon,\delta)$-PAC learnable in just $O(r \epsilon^{-2} \log(n/\delta))$ samples when $k > 2$ (Thm. \ref{thm:ub_ordbr}, Sec. \ref{sec:br_ub}). This improves over the known sample complexity bound of $\tilde{O}(n\epsilon^{-2} \log(1/\delta))$ of the case where the arms are independent when $r \ll n$. 
 
(4). We complement our upper bound with a matching lower bound (up to logarithmic factors), justifying the tightness of our analysis (Thm. \ref{thm:lb_ordbr}, Sec \ref{sec:br_lb}). 

(5). We also show a lower bound of $\Omega(n\epsilon^{-2}\log(1/\delta))$ (Thm. \ref{thm:lb_duel}, Sec. \ref{sec:br}) when the learner is forced to play just pairwise queries ($k = 2$), which indicates how playing larger subset sizes allows the learner to exploit the underlying correlation structure achieving faster learning rates. In contrast, however, a pairwise query model $(k=2)$ fails to exploit the underlying correlation structure as shown in Thm. \ref{thm:lb_ordbr}.

(6). Finally we extend our analysis to a general $\eta$-`noisy-Block-Rank' based RUM choice model justifying \emph{robustness} of proposed method which shows its $O(r \epsilon^{-2} \log(n/\delta))$ sample complexity performance remains unaffected under some `tolerable $\eta$-noise' in the correlation structure even if the underlying correlation matrix becomes full rank, i.e. $r=n$ (Sec. \ref{sec:nbr}). 

This work is mostly theoretical in nature and in particular has no societal impact.



{\bf Related Works.} For the classical multiarmed bandits setting, there is extensive literature on PAC-arm identification problem \citep{Even+06,Audibert+10,Kalyanakrishnan+12,Karnin+13,LilUCB}, where the learner gets to see a noisy draw of absolute reward feedback of an arm upon playing a single arm per round. 
On the contrary, learning to identify the best item(s) with only relative preference information (ordinal as opposed to cardinal feedback) has seen steady progress since the introduction of the dueling bandit framework \citep{Zoghi+13} with pairs of items (size-$2$ subsets) that can be played, and subsequent work on generalization to broader models both in terms of distributional parameters \citep{Yue+09, Adv_DB,Ailon+14, Zoghi+15MRUCB} as well as combinatorial subset-wise plays \citep{MohajerIcml+17,pbo,SG18,Sui+17}. 

There have been a few works in the MAB literature to exploit the advantages of item correlations, which assume the knowledge of the correlation structure (in terms of side information or online feedback-graphs). \cite{SideInfo11,SideInfo14,SideInfo16,Alon+15,Alon+17} study the MAB problem assuming a relation graph over the nodes, however their setting also requires revealing rewards of the neighboring set of the pulled arm, which reduces this to a semi-bandit (side information) setting. On the contrary, our setting is based on a pure bandit feedback model that reveals only a noisy reward of the selected arm. \cite{CheapBandits} also consider a stochastic sequential learning problems on graphs but here the learner gets to observe the average reward of a group of graph nodes rather than a single one. 
\cite{bok16} studies the top $k$ item determination problem of multiarmed bandits for correlated arm rewards (where the underlying correlation structure can be arbitrary) and show that in the worst case the learner could be forced to consider all $\Omega\left(n \choose k\right)$ subsets.
 \cite{Singh20,Gupta19} studies the MAB regret minimization problem under correlated arms, modeling the reward dependencies in terms of clusters or some known correlation structures. 
To the best of our knowledge there have been no previous attempts towards understanding how item correlations affect the sample complexity of the winner determination problem in preference bandits, \emph{specifically in settings where the learner has no prior knowledge of the underlying correlation}, which is the primary focus of the current work. We believe that this is a new direction which can be explored along multiple fronts.

\vspace{-5pt}
\section{Preliminaries}
\label{sec:prelims}
{\bf Notations.} We denote by $[n]$ the set $\{1,2,...,n\}$. 

\vspace{-0pt}

\subsection{Low Rank Subset Choice Models (accounting Item Correlations)}
\label{sec:corrRUM}

Before introducing our \lr, we recall the definition of the standard (independent) \emph{discrete random utility based choice models} (RUMs) \cite{Az+12,ChenSoda+18} used in the preference bandits literature, which however do not take into account the item correlations.  

\textbf{Discrete Random Utility based Choice Model (RUMs). } 
 RUMs are a widely-studied class of discrete choice models; they assume a (non-random) ground-truth utility score $\mu_{i} \in \R$ for each alternative $i \in [n]$, and assign a distribution $\cD_i(\cdot|\mu_{i})$ for scoring item $i$, where $\E[\cD_i \mid \mu_i] = \mu_i$. To model a winning alternative given any set $S \subseteq [n]$, one first draws a random utility score $X_{i} \sim \cD_i(\cdot|\mu_{i})$ for each alternative in $S$, and selects an item with the highest random score. More formally, the probability that an item $i \in S$ emerges as the {\em winner} in set $S$ is given by:
\vspace{-1pt}
\begin{align}
\label{eq:prob_rum}
Pr(i|S) = Pr(X_i > X_j ~~\forall j \in S\sm \{i\} ),
\end{align}
where ties are broken uniformly over all elements in set $S$. 
It is generally assumed that for each item $i \in [n]$, its random {\em utility score} $X_i$ is of the form $X_i = \mu_i + \zeta_i$, where all the $\zeta_i \sim \cD$ are `noise' random variables drawn \emph{independently} from a probability distribution $\cD$. 
For the purposes of analysis, it is generally assumed without loss of generality\footnote{Under the assumption that the learner's decision rule does not contain any bias towards a specific item index}, that $\mu_1 > \mu_i \, \forall i \in [n]\setminus\{1\}$ for ease of  exposition\footnote{The extension to the case where several items have the same highest parameter value is easily accomplished.}. Formally, we define the \emph{best-item} to be one with the highest score parameter: $i^* \in \underset{i \in [n]}{\text{argmax}}~\mu_i = \{1\}$, under the assumptions above. We will denote the model as \indrum ~(\irum) for the rest of the paper. 

\textbf{Popular examples of \irum.} A widely used RUM is the {\it Multinomial-Logit (MNL)} or {\it Plackett-Luce model (PL)}, where the $\cD_i$'s are taken to be independent Gumbel$(0,1)$ distributions with location parameters $0$ and scale parameter $1$ \citep{Az+12}, which results in score distributions $Pr(X_i \in  [x,x + dx]) = e^{-(x - \mu_i)}e^{-e^{-(x - \mu_i)}} dx$, $\forall i \in [n]$. Similarly, other different families of discrete choice models can be considered for different choices of the underlying iid noise model $\zeta_i \sim \cD$, e.g.  Exponential, Uniform, Gaussian, Weibull etc \cite{SG20}. 

\textbf{Limitations of existing results for \irum.} The $(\epsilon,\delta)$-PAC best-arm identification problem under this model has already been studied in the literature. In particular, ~\cite{SG18,SG20} show a fundamental sample complexity lower bound of ${\Omega}(\frac{n}{\epsilon^2}\ln \frac{1}{\delta})$ for this, given fixed $\epsilon,\delta \in (0,1)$. A disappointing takeaway from these results are that the bounds are subset size independent, which triggers the natural question: Why should one play subsets of larger sizes if it does not lead to a faster learning rate? One can simply get away with the problem of identifying the $\epsilon$-best item by just playing pairwise preference games $(k = 2)$ in that case. 
%

\textbf{Main questions: Can we exploit reward correlations in Preference Bandits (with $k$-subsetwise pulls) without the knowledge of the underlying correlation structure?} Naturally the first question is if we incorporate utility correlations in $(X_1,\ldots X_n)$ does that lead to faster learning rate? Further, what is the right measure of correlation in a choice model? 

\begin{rem}[Why MAB setup can not exploit reward correlations]
Note in the setting of standard MAB, the item correlations do not play any role in improving the learning rate beyond $\Omega(\frac{n}{\epsilon^2}\ln \frac{1}{\delta})$ \cite{Even+06}. This is due to the inherent limitation of models which restrict the learner to query feedback of just single arms at every round -- this means irrespective of the correlation model $\Sigma$, the learner would never have a way to distinguish if two arms are fully correlated or exactly identical from single arm pulls.
\end{rem}


\textbf{Our proposed preference-choice models (to capture correlations). } Towards this we study the following natural generalization of \irum:
Define the utility score vector $X$ as a \emph{multivariate} random variable of the form: $X = \bmu + \boldsymbol \zeta$, where $\bzeta \sim \cD$ is a multivariate noise drawn from a {\em joint} distribution $\cD$ (instead of sampling the $n$ utility scores $X = (X_1,\ldots,X_n)$ independently), such that it has mean zero and correlation structure quantified by the $n\times n$-matrix of correlation coefficients $\Sigma:={\rm Corr}(\boldsymbol \zeta)$. Again, the choice probability $P(i|S)$ of any arm $i \in S$ in any subset $S \subseteq [n]$, can still be defined same as that suggested in Eqn. \eqref{eq:prob_rum}. We refer to this model as \emph{Correlated-RUM-based-Choice-Models}. For example, assume $\cD = \cN(\0,\Sigma)$---a multivariate zero mean Gaussian noise with some fixed (but unknown) covariance matrix $\Sigma$. In particular when $\Sigma$ is not the identity matrix, this recovers a `\emph{correlated Gaussian based choice model}'.

{\bf \lr ~(\rumr). }
A first step towards understanding the effect of item correlations on learning rate is to determine a suitable `measure of correlation' among the utility scores $X_1,\ldots,X_n$ in terms of some properties of the underlying correlation matrix $\Sigma$. Formally $\Sigma$ is a $n \times n$ matrix such that $\Sigma(i,j) := {\rm Corr}(\zeta_i,\zeta_j)$ where ${\rm Corr}(\cdot,\cdot)$ is used to denote the correlation coefficient{\footnote{We point that since correlation is translation invariant, we have ${\rm Corr}(\zeta_i,\zeta_j) = {\rm Corr}(X_i,X_j)$ for every $i,j \in [n]$, and hence it suffices to work with the correlation matrix defined on the $\zeta$-variables.}}.
A natural quantity to express the complexity of item correlations is the rank of the underlying correlation matrix. E.g. PCA (\cite{bishop06}) or Matroids (\cite{ox06}) which has varied applications in many real systems including image, graph networks or information retrieval. Motivated by these, we define the \lr, which models the item dependencies through rank of the correlation matrix $\Sigma$. Precisely we assume $\rank(\Sigma)=r$ for some $r \in [n]$. Clearly the setting $r=n$ expresses independent discrete RUM based choice model as a special case. We will henceforth denote this model by `$r$-\lr' or \rumr.

\textbf{$r$-\br. } An interesting special case of low-rank choice model is one where the item correlations results in an $r$-clustering (referred to as `blocks' henceforth) of the set of items. In particular, we say that a $r$-\br\, instance has {\em block-rank} $r$ if there exists a partitioning on the set of items $[n]$ into blocks $\cB_1 \uplus \cB_2 \uplus \cdots \uplus \cB_r$ such that the following properties hold:
\begin{itemize}
	\item[(i)] {\it Inter-block Identity}: For any block $\cB_i$ and any pair of items $a,b \in \cB_i$, we have $\zeta_a = \zeta_b$.
	\item[(ii)]{\it Cross-block Independence}: For any subset of items $S \subseteq [n]$ such that $|S \cap \cB_i| \leq 1$ for every $i \in [r]$, the set of variables $(\zeta_i)_{i \in S}$ is jointly independent.
\end{itemize}
Note that in this setting, the correlation matrix admits a block diagonal structure i.e.,
$$\Sigma(a,b) = 
\begin{cases}
1,~ \forall a,b \in \cB_i, ~i \in [r],\\
0, ~\forall a \in \cB_i, b \in \cB_{j}, ~i,j \in [r],~i\neq j.
\end{cases}
$$
Again for brevity, we shall refer to this model as \rumb\, in the rest of the paper.

\textbf{(Noisy) $(r,\teta,\eta)$-\br. } The $(r,\teta,\eta)$-\br~model generalizes $r$-\br~by allowing intra-block variables to be {\em almost correlated} and inter-block variables to be nearly independent (for some $\teta, \eta \in (0,1]$). Formally, in $(r,\teta,\eta)$-\br~instance, the set of items $[n]$ admits a partitioning into blocks $\cB_1 \uplus \cdots \uplus \cB_r$ such that the following properties hold.
\begin{itemize}
	\item[(i)]{\it Cross-block Approximate Independence}: For any subset of items $S \subseteq [n]$ such that $|S \cap \cB_i| \leq 1$ and any non-trivial partitioning $S = S_1 \uplus S_2$ we have $I(\zeta_{S_1};\zeta_{S_2}) \leq \teta$ where $\zeta_{S_1} := (\zeta_i)_{i \in S_1}$ denotes the set of variables in $S_1$ and  $I(\cdot,\cdot)$ denotes the mutual information (MI) between two (sets of) variables. 
	\item[(ii)] {\it Inter-block Approximate Identity}: For any block $\cB_i$ and any pair of items $a,b \in \cB_i$ we have ${\rm Corr}(\zeta_i,\zeta_j) \geq 1 - \eta$.
\end{itemize}
Note that instantiating $\teta, \eta = 0$, we recover the $r$-\br~setting as a special case. For {\em Gaussian noise}, one can express both items (i) and (ii) above in terms of correlation, since it is folklore that a pair of Gaussians can be uncorrelated if and only if they are independent.

\vspace*{-10pt}

\section{Problem Setting}
\label{sec:prob}
\vspace*{-7pt}
We consider the \emph{Probably Approximately Correct (PAC)} version of the best-arm identification problem through subset-wise comparisons. %
%
Formally, the learner is given a finite set $[n]$ of $n > 2$ items or `arms' along with a playable subset of size $k \leq n$. At each round $t = 1, 2, \ldots$, the learner selects a subset $S_t \subseteq [n]$ of size at most $k$ 
distinct items, and receives (stochastic) feedback of the `winning item' drawn according to $Pr(\cdot \mid S_t)$ (see Eqn. \eqref{eq:prob_rum}) depending on (a) the chosen subset $S_t$, and (b) a \rumr\, choice model with parameters $\bmu = (\mu_1,\mu_2,\ldots, \mu_n)$ a priori unknown to the learner. 


\subsection{Correctness and Sample Complexity: \prob} 
\label{sec:obj}


For a \lr\, \rumr\, instance with $n \geq k$ arms, an arm $i \in [n]$ is said to be $\epsilon$-optimal if $\mu_i > \mu_1 - \epsilon$. A sequential learning algorithm that depends on feedback from an appropriate subset-wise feedback model is said to be $(\epsilon,\delta)$-{PAC}, for given constants $0 < \epsilon \leq \frac{1}{2}, 0 < \delta \leq 1$, if the following properties hold when it is run on any instance \rumr:
%
(a) it stops and outputs an arm $I \in [n]$ after a finite number of decision rounds (subset plays) with probability $1$, and (b) the probability that its output $I$ is an $\epsilon$-optimal arm in \rumr\, is at least $1-\delta$, i.e, $Pr(\text{$I$ is $\epsilon$-optimal}) \geq 1-\delta$. 
By {\em sample complexity} of the algorithm, we mean the expected time (number of decision rounds) taken by the algorithm to stop when run on the instance \rumr.
%


\section{Impossibility Result: General \lr}
\label{sec:genrank}

In this section we show that allowing for arbitrary correlations can make the $\epsilon$-optimal winner determination problem ill defined in the following sense: one can construct instances where there are subsets for which the item with the largest win probability is not the same as the item with the largest score. In particular, such instances can be simply constructed in a way such that the covariance matrix is just $2$-dimensional. We state the observation formally in the following lemma.

\begin{restatable}[]{lem}{lbgen}
\label{thm:lb_gen}
\label{THM:LB_GEN}
Consider an instance of \rumr\, with $n=k$, and $r = 2$. Then for any even $k \geq 4$ and $0 < \epsilon \leq \epsilon(k)$, there exist scores $\mu_1 > \mu_2 + \epsilon \geq \cdots \geq \mu_k$ with underlying correlation matrix $\Sigma \in \R^{k \times k}$ s.t. 
$
\argmax_{i \in [k]} \Pr\left( i | [k] \right)  \neq 1;
$
i.e. Item-$1$, despite of being the only $\epsilon$-optimal item, would not have the maximum winning probability when played in a subset. 
\end{restatable}

Clearly, the above kind of instances are a {\em structural} barrier (as opposed to an information theoretic one) to the winner determination problem, since if the instance itself is the subset equipped with the distribution from the above lemma, the observations will be guided by the win probabilities, which do not favor the true winner. 
%

\vspace*{-8pt}
\begin{proof}[Proof Sketch of Lem. \ref{thm:lb_gen}]

Consider the following generic way of constructing a family of correlated Gaussians using unit vectors ${\bf v}_1,\ldots,{\bf v}_k$. (i). Sample a random Gaussian vector ${\bf g} \sim \cN({\bf 0},{\bf I}_{2 \times 2})$. (ii). For every $i \in [k]$, set $g_i = \langle {\bf v_i},{\bf g} \rangle$. 

We use the above geometric interpretation to define the correlation structure of the Gaussians. For every $i \in [k]$, we set ${\bf v}_i := {\bf u}(\alpha_i)$, where ${\bf u}(\alpha)$ is the unit vector $(\cos \alpha, \sin \alpha)$. We define the corresponding $\alpha_i$'s as follows. We set $\alpha_1 = 0$, $\alpha_k = \pi$ and for every $i \in \{2,\ldots,k-1\}$, we set $\alpha_i = (-1)^{i~{\rm mod}~2} \cdot \pi/4$. 

Finally, we assign the score vector ${\bmu} = (\mu_1,\ldots,\mu_k)$ as follows. We set $\mu_1 = \mu + \epsilon$ and $\mu_j = \mu$ for every $j \in [k] \setminus\{1\}$. Note that in the above construction of $(g_1,\ldots,g_k)$, the correlation matrix $\Sigma$ is exactly ${\bf V}^\top{\bf V}$ where ${\bf V} := [{\bf v}_1,\ldots,{\bf v}_k]$. Since ${\bf v}_i$ are $2$-dimensional unit vectors, we have ${\rm rank}(\Sigma) \leq 2$. Furthermore, arm $1$ is the only $\epsilon/2$-best arm in the setting.

{\bf Analysis.} 
We first observe that when $\epsilon = 0$ (i.e., all items are assigned score $\mu$), the win probability of an item $i$ when $[k]$ is played is exactly the angular measure of arc consisting of the points on the unit circle closest to vector ${\bf v}_i$. 
In that case, one can easily verify that 

\vspace{-20pt}
\begin{align*}
\Pr_{{\bmu}' = (\mu,\ldots,\mu)}\left(1 | [k]\right) = 1/8  \text{ and } 
Pr_{{\bmu}' = (\mu,\ldots,\mu)}\left(k|[k]\right) = 3/8
\end{align*}
\vspace{-20pt}

Furthermore, even when $\epsilon$ is non-zero but small enough as a function of $k$, using a first order approximation argument, for the actual score vector ${\bmu} = (\mu + \epsilon,\mu,\ldots,\mu)$, we have the following win probability bounds: $\Pr_{{\bmu} = (\mu + \epsilon,\ldots,\mu)}\left(1 | [k]\right) \leq 1/8 + o(1), ~
\Pr_{{\bmu} = (\mu+ \epsilon,\ldots,\mu)}\left(k|[k]\right) \geq 3/8 - o(1)$

\vspace{-10pt}
	\begin{figure}[H]
		\begin{center}
			\includegraphics[width=0.4\textwidth]{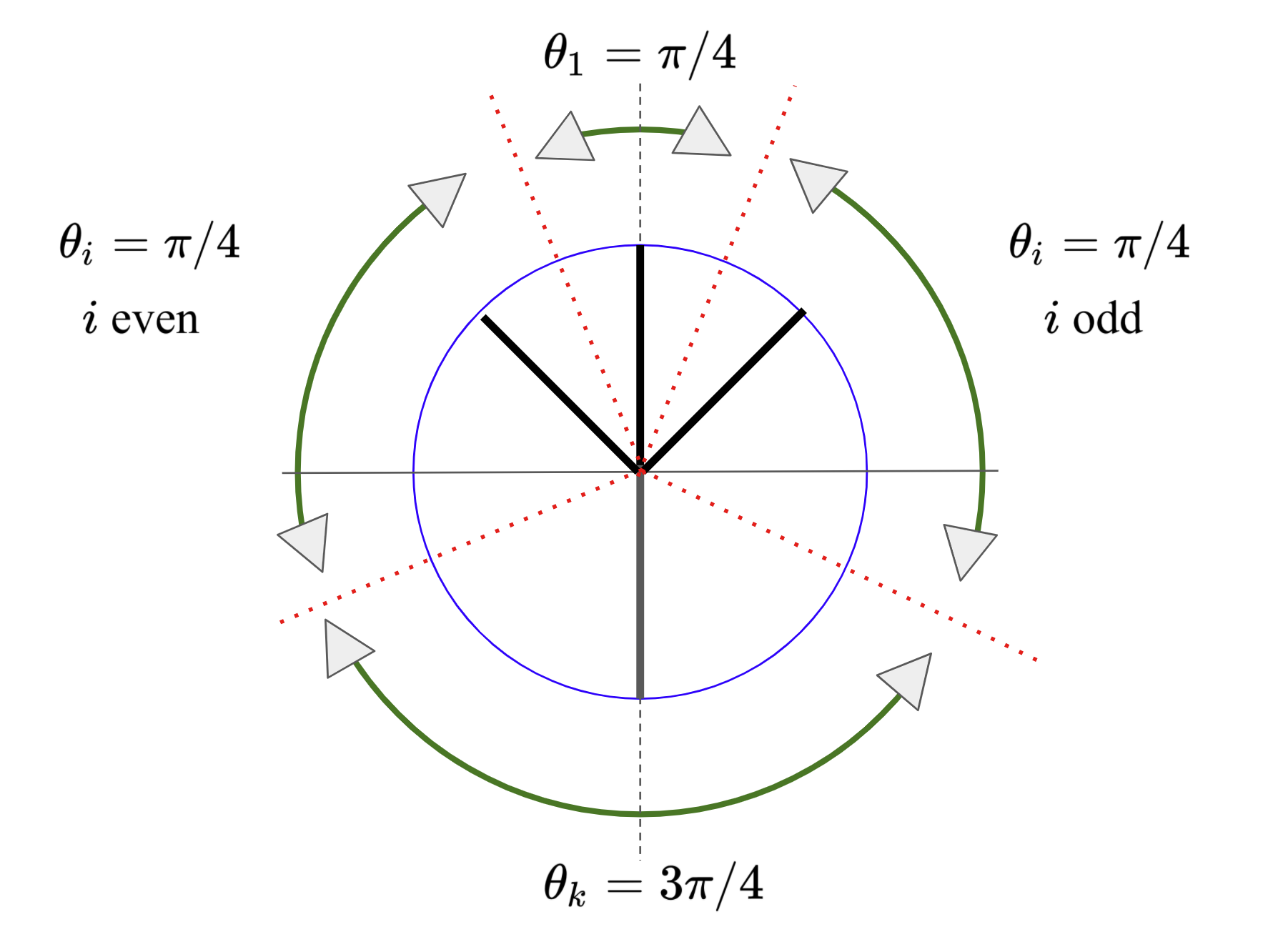} 
			\vspace{-7pt}
			\caption{Winning Sectors corresponding to the arms}		
			\label{fig:imposs}
		\end{center}
	\end{figure}
	
\vspace{-20pt}

In summary, we have $\mu_1 > \mu_{k} + \epsilon/2$ but $\Pr(1|[k]) < \Pr(k|[k])$, which establishes the guarantees claimed. We include the full proof in Appendix \ref{app:gen-rank}.  
\end{proof}

\section{$r$-\br\, Choice Model}
\label{sec:br}
\label{SEC:BR}

The impossibility result for the general \br\, case (Sec. \ref{sec:genrank}) motivates us to understand if a faster learning rate can be achieved through imposing more structured item correlations. In particular, in this section, we use a more combinatorial notion of measure of simplicity (namely, block rank) to explore $r$-\br~ instances (see Sec. \ref{sec:prelims} for description). In particular, our contributions include an  $O(r\epsilon^{-2}\log(n/\delta))$-sample complexity algorithm for PAC learning for \rumb~instances when the learning algorithms is allowed to play subsets of sizes at least $3$, and complement it with matching sample complexity lower bound for the same setting. In addition, we show a $\Omega(n\epsilon^{-2}\log(1/\delta))$-sample complexity lower bound for these instance when the learner is restricted to play just pairwise duels. 

\vspace{-5pt}
\subsection{Algorithm: Sample Complexity Bound} 
\label{sec:br_ub}

We first design an algorithm for this setup based on the following key intuition. For any instance with block rank $r$, the information theoretic bottleneck here is the winner determination problem among the best item from each of the $r$-blocks. However, the challenge here is the obvious one, the identities of these items are not known upfront, and as such, any off-the-shelf algorithm for the $(\epsilon,\delta)$-PAC learning problem which does not exploit the underlying correlation structure, would essentially end up solving the winner determination problem on $n$-arms leading to a sample complexity of $O(n/\epsilon^2 \log 1/\delta)$. 

\textbf{Main ideas:} We circumvent these issues by: 

$(1)$ \emph{Fast Pre-processing step} (with number of arm pulls independent of $\epsilon$) which reduces the effective pool of candidate items to a subset of size at most $r$: The pre-processing step is based on the following principle. Given any non ``strictly optimal''\footnote{i.e., an item whose score is not strictly larger than those of every other item in the same block.} item $i$ within a block, we can always find another item $i'$ in the same block whose win probability is at least as large as that of $i$ on any subset $S$ which simultaneously contains $i$ and $i'$.  On the other hand, since $1$ is the unique winner, it's win probability is never dominated by that of another item.  This observation is the core guiding principle for  our design of the pre-processing step which plays all possible triples and eliminates items based on their worst case win probability estimates. In particular, with high probability it returns a set of at most $r$-arms, say $S$, each of which belongs to a distinct block, and one of which is the optimal arm. Since they come from the $r$ distinct blocks, they are independent. 

$(2)$ Now to obtain an $\epsilon$-best item, we simply run the {\emph Sequential-Pairwise-Battle} algorithm\footnote{See Appendix \ref{app:seq-pb} for an informal self-contained description of the algorithm.} of \cite{SG20} (precisely Seq-PB$(S,\min(k,r),\epsilon,\delta/2,c(\cD))$, which is known to be a provably optimal $(\epsilon,\delta)$-PAC for any \irum\, given any underlying noise model distribution $\cD$ ($c(\cD)$ being a constant depending on the `minimum-\ratio' (\mrat) of $\cD$, see Defn. \ref{def:rat}). {Note that our algorithm is adaptive and does not require prior knowledge of the block-rank $r$.} The pseudocode is given as Algorithm \ref{alg:ordbr}.

\begin{center}
	\begin{algorithm}[h]
		\caption{\algbr \, (\brpb)} 
		\label{alg:ordbr}
		\begin{algorithmic}[1]
			\STATE {\bfseries Input:} 
			\STATE ~~~ Set of items: $[n]$. Error bias: $\epsilon >0$, Confidence parameter: $\delta >0$. 
			\STATE ~~~  Noise model $(\cD)$ (or equivalently $c(\cD)$, a noise model dependent constant)
			\STATE {\bfseries Initialize:} 
			\STATE ~~~ $t \gets O\left(\log\frac{4n^3}{\delta}\right)$ and set ${\rm Flag}(i) \gets 0$ for every $i \in [n]$.
			\FOR {$\cT \in {[n] \choose 3}$} 
			\STATE ~~~ Play the triple $\cT$ for $t$-times. For $i \in \cT$, let $N_\cT(i)$ be the number of times $i$-wins.
			\STATE ~~~ For every item $i \in \cT$ such that $N_{\cT}(i) \leq 0.26t$, mark ${\rm Flag}(i) \gets 1$.
			\ENDFOR
			\STATE ~~~ Construct set $S := \{ i \in [n] | {\rm Flag}(i) = 0 \}$.
			\STATE ~~~ Find: $\hat i \leftarrow$ Seq-PB$(S,\min(k,|S|),\epsilon,\delta/2,c(\cD))$ (Alg. 1, \cite{SG20}). 
			\STATE ~~~ Output $\hat i$: The winner returned by Seq-PB.
		\end{algorithmic}
	\end{algorithm}
	\vspace{-2pt}
\end{center}
\vspace{-10pt}

The following theorem formally states the guarantee of the above algorithm.

\begin{restatable}[Alg. \ref{alg:ordbr}: Correctness and Sample Complexity for \rumb]{thm}{ubordbr}
\label{thm:ub_ordbr}
\label{THM:UB_ORDBR}
Consider any \rumb\, \br\, choice model with noise distribution $\cD$, $k >2$. Then, Alg. \ref{alg:ordbr} is $(\epsilon,\delta)$-{PAC} with sample complexity $\max \left\{O(n^3\log(n/\delta)), O(\frac{r}{c \epsilon^2} \ln \frac{r}{\delta})\right\}$, where $c:= c(\cD)$ is a constant depending on $\cD$.
\end{restatable}

\begin{rem}[Improved Sample Complexity] 
	Note that above implies improved sample complexity of $O(r\epsilon^{-2} \log(r/\delta))$ which is much smaller than the usual bound of $O(n\epsilon^{-2} \log (r/\delta))$, when $r \ll n$. This is due to the fact that in general, block rank is a more precise notion of the effective number of arms to be considered for the winner determination problem.
\end{rem}

\begin{rem}[Parameter regime for improved Sample Complexity]
The above algorithm exhibits improved sample complexity $O(r\epsilon^{-2}\log(n/\delta))$ for all  $\epsilon \in (0, (r/n^3)^{1/2}]$; this improves on the $O(n\epsilon^{-2}\log(n/\delta))$ of~\cite{SG20} for the \irum~model.  In particular, we don't need $n$ to be constant for the overall sample complexity to be $O(r \epsilon^{-2} \log(r\delta))$ i.e., it is actually the trade-off between $r/n$ and $\epsilon$ that determines the regime of parameters under which Algorithm \ref{alg:ordbr} exhibits improved convergence rates.
\end{rem}

\begin{rem}
In particular, \cite{SG20}~show that $c(\cD)$ is a constant for several popular choices of noise distributions such as {\em Uniform}, {\em Gumbel}, {\em Gaussian}, {\em Gamma}, {\em Weibull} (see). Consequently, our algorithm \emph{\algbr}\, gives a $O(\frac{r}{ \epsilon^{2}} \log(n/\delta))$-sample complexity guarantee for all such distributions as shown in the Thm. \ref{thm:ub_ordbr}. 
\end{rem}

\begin{proof}[Proof Sketch of Thm. \ref{thm:ub_ordbr}]
	{\bf Justifying Correctness.} It is based on the following idea that when items are played in triples, \emph{there exists a separation in worst case win-probabilities} (when played in triples) between non-strictly optimal items and the best item.  
	\begin{restatable}[]{cl}{wprob}			\label{cl:win-prob1}
		For any triple $\cT = (1,i,j)$, we have $Pr(1 | \cT) \geq 1/3$. 
	\end{restatable}
	\vspace{-10pt}
	\begin{restatable}[]{cl}{wprobb}				\label{cl:win-prob2}
		For any triple $\cT = (1,i,j)$, such that $i,j$ are from the same block, if $\mu_i \geq \mu_j$, then $\Pr\left(j|\cT\right) \leq 1/4$.
	\end{restatable}
		\vspace{-5pt}
	The first two claims taken together imply: (a) every item $j$ whose score $\mu_j$ is not uniquely largest in its block participates in a triple where it wins with probability at most $1/4$ and (b) in every triple the best item $1$ wins with probability at least $1/3$. Since our choice of number of trials $t$ is large enough, we know that the empirical estimates are close enough approximates of the true win probabilities; points (a) and (b) also hold for the empirical win-probability estimates $\{N_\cT(i)/t\}_{i,\cT}$. This observations are stitched together in the following lemma which gives useful characterization of items which are flagged inside the for loop.
	
	\begin{restatable}[]{lem}{flag}
	\label{lem:flag}
		With probability at least $1 - \delta/2$, the following holds. For any item $j \in [n]$, if there exists an item $i$ from the same block for which $\mu_i \geq \mu_j$, we have ${\rm Flag}(j)  = 1$. Furthermore, ${\rm Flag}(1) = 0$.
	\end{restatable}
	In particular, with high probability, every item $j$ satisfying the premise of (b) gets flagged, whereas item $1$ never gets flagged. And whenever this high probability event holds, the resulting set $S$ must consist of at most $r$-arms, which are all independent and $1 \in S$.
	\begin{restatable}[Pre-processing Step Guarantee]{lem}{preproc}					
		\label{lem:pre-proc}
		With probability at least $1 - \delta/2$, the set $S$ satisfies the following conditions. For every $i \in [r]$, $|S \cap \cB_i | \leq 1$. Additionally $1 \in S$.
	\end{restatable}
	Now assume that the subset $S$ satisfies the guarantees of the above lemma. Since the items in $S$ come from distinct blocks, the corresponding arms are independent and the subset-wise feedback on subsets of $S$ follow the independent RUM model. Hence running Algorithm 1 from \cite{SG20} will return an $\epsilon$-best arm in $O(r\epsilon^{-2} \log(r/\delta))$-samples.
	
	{\bf Justifying the Sample complexity.} In the for loop (Lines 6-9), each triple $\cT \in {n \choose 3}$ is played $t$ times. Therefore, then total number of arm pulls in the for loop is bounded by $O(n^3 t) \leq O(n^3 \log(n/\delta))$ which is a constant independent of $\epsilon$. Therefore, step corresponding to Line $11$ incurs a sample complexity cost of $O(r{\epsilon^{-2}}\log(r/\delta))$. Since the latter term dominates as $\epsilon \to 0$, this establishes the desired sample complexity. The complete proof is given in Appendix \ref{app:br_ub}.
\end{proof}

\begin{rem}
We consider playing subsets of various sizes, because without this relaxation, the winner determination problem can again become ill defined in the correlated setting (see Lem. 	\ref{thm:impos_fxd}).
\end{rem}

\vspace{-10pt}
\subsection{\textbf{$r$-\br:} Lower Bound} 
\label{sec:br_lb}
Our lower bound analysis is based on the following intuition: Given an instance $\cI$ with $r$-\br, where $[n] = \uplus_{i \in [r]} \cB_i$ is the partitioning of arms into block structure. Consider the set $\cS$ constructed by adding the arm with the highest score from each block $\cB_i$. Now the key insight is that any algorithm which solves the $\epsilon$-best arm identification problem on $\cI$ must also solve the $\epsilon$-best arm identification problem on the set of independent arms $\cS$. This observation can be used to embed instances of \texttt{IND-RUM}$(r,k,\bmu)$ into instances of \texttt{BR-RUM}$(n,k,r,\bmu')$, thus forcing the worst case sample complexity of the latter to be lower bounded by that of the former, which is known to be $\Omega(r\epsilon^{-2} \log(1/\delta))$.

\vspace*{-13pt}
\begin{figure}[H]
	\begin{center}
		\includegraphics[width=0.4\textwidth]{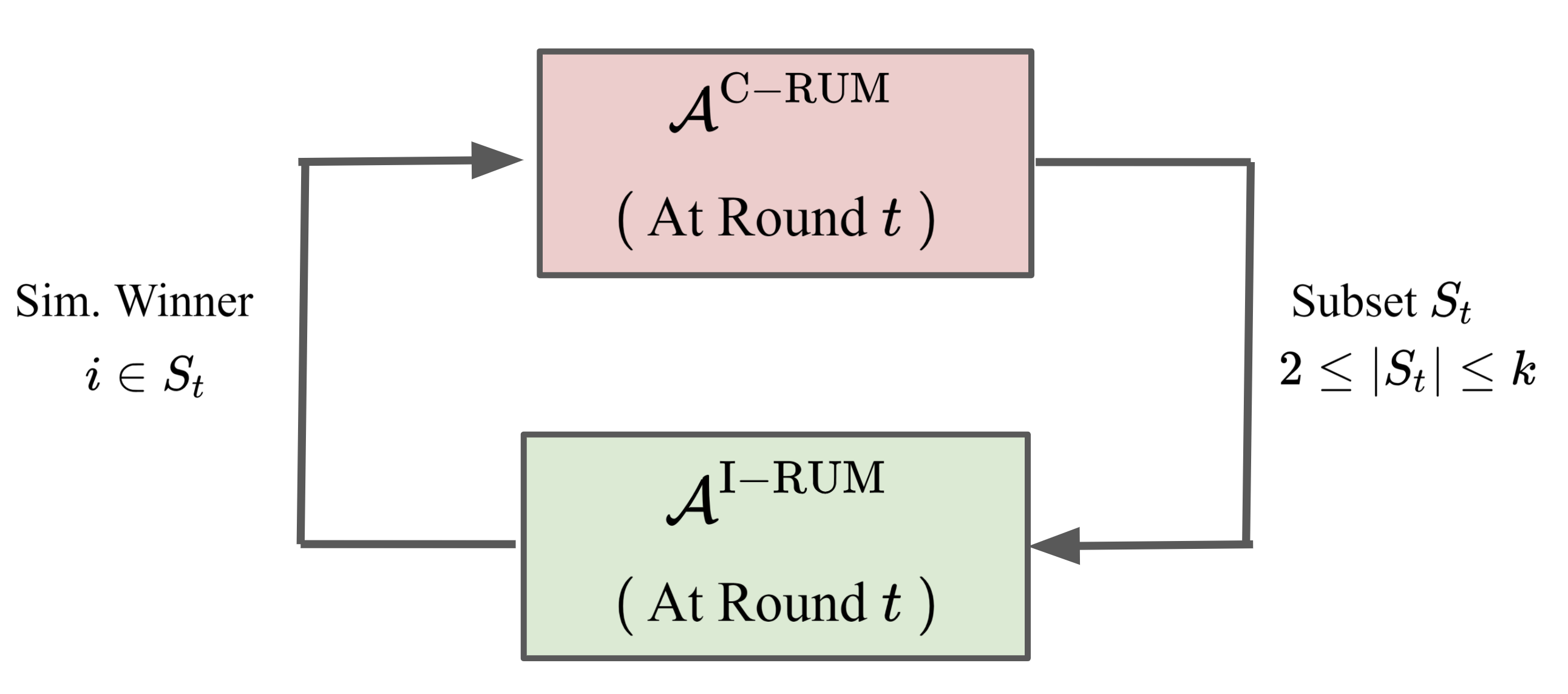} 
		\vspace{-7pt}
		\caption{Reduction (Pseudocode in Appendix \ref{app:pc2})}		
		\label{fig:lb_demon1}
	\end{center}
\end{figure}
\vspace*{-20pt}

\begin{restatable}[Performance limit for \obr]{thm}{lbordbr}
	\label{thm:lb_ordbr}
	\label{THM:LB_ORDBR}
	Given $\epsilon \in (0,1]$, $\delta \in (0,1]$, $r,k \in [n]$, for any $(\epsilon,\delta)$-PAC algorithm $A$ for the \prob, there exists an instance of \rumb, say $\nu$, where the expected sample complexity of $A$ on $\nu$ is  at least
	$\Omega\big( {r}\epsilon^{-2} \log 1/\delta\big)$.
\end{restatable}

\begin{proof}[Proof sketch of Thm. \ref{thm:lb_ordbr}]
The proof of Theorem \ref{thm:lb_ordbr} uses a reduction from the problem of $(\epsilon,\delta)$-PAC learning in an \texttt{IND-RUM}$(r,k,\bmu)$ instance to an $(\epsilon,\delta)$-PAC learning problem in a \texttt{BR-RUM}$(n,k,r,\bmu')$ instance. Formally, the reduction proceeds as follows. Given an algorithm $\cA^{C-RUM}$ which $(\epsilon,\delta)$-PAC learns best items from \texttt{BR-RUM}$(n,k,r,\bmu')$ instances, we can construct an algorithm $\cA^{\rm I-RUM}$ which does the same for the independent setting with $r$-arms. In particular, the algorithm embeds the best item learning problem on $r$-arms over a unknown score profile $\bmu$ inside best item learning problem on $n$-arms with correlations, and then uses $\cA^{\rm C-RUM}$ to solve the large problem. This is done using a simple idea: the outer algorithm $\cA^{\rm C-RUM}$ adds $n-r$ dummy items to the set with score $\mu_i = -\infty$ with appropriate correlation structure. This ensures that:
(i) 	The $\epsilon$-best item set in $[r]$ is also the $\epsilon$-best item set in the larger set $[n]$.
(ii) The algorithm $\cA^{\rm I-RUM}$ can simulate the subsetwise preference feedback required by $\cA^{\rm C-RUM}$ on $[n]$ using its own preference feedback on subsets of $[r]$ (Fig. \ref{fig:lb_demon1}).

Overall, if $\cA^{\rm C-RUM}$ is an $(\epsilon,\delta)$-PAC algorithm for \texttt{BR-RUM}$(n,k,r,\bmu')$, then so is $\cA^{\rm I-RUM}$ for \texttt{IND-RUM}$(r,k,\bmu)$. Therefore the sample complexity of $\cA^{C-RUM}$ is bounded by that of $\cA^{I-RUM}$, which we prove to be $\Omega(r \epsilon^{-2}\log(1/\delta))$--this is done by extending previous known lower bounds for fixed subset sizes to the setting of variable-sized subsetwise plays (Thm. \ref{thm:lb_ind}, Appendix \ref{app:lb_ind}). The proof is given in Appendix \ref{app:br_lb}.
\end{proof}

On the other hand, our next result shows that, there is no advantage in querying pairwise-feedback $(k=2)$ even for any $r \ge 2$ (note $r=1$ is a trivial case), as stated formally in the following theorem.

\begin{restatable}[Pairwise Preferences: Sample Complexity Lower Bound for \lr]{thm}{lbduel}
	\label{thm:lb_duel}
	\label{THM:LB_DUEL}
	Given $\epsilon \in (0,1/4]$, $\delta \in (0,1]$, and general $r \in [n]$ $(r> 1)$, for any $(\epsilon,\delta)$-PAC algorithm $A$ for \prob\, problem, $\exists$ an instance of \texttt{BR-RUM}$(n,2,r,\bmu)$, where the expected sample complexity of $A$ is at least $\Omega\big( \frac{n}{\epsilon^2} \log \frac{1}{\delta}\big)$-- independent of $r$.
\end{restatable}

\begin{rem}
[Separation in the sample complexity for $k=2$ vs $k=3$] Intuitively, triples can be used to determine whether a subset involves the winner much faster: Consider the instance with $2$ blocks, where the first block is a singleton with score $\mu + \epsilon$ and the second block consists of $(n-1)$.identical arms, each with score $\mu$. Then for every distinct choice of $i,j \in [n]$ we have  $\Pr(i|\{i,j\}) = \frac12 + O(\epsilon)$ if $i = 1$ and $\Pr(i| \{i,j\}) = \frac12$ if $i\neq 1$ i.e, the duels involving the winner behave near identically to duels not involving it and it would take $\Omega_\delta(\epsilon^{-2})$-queries\footnote{Here we use $\Omega_\delta(\cdot)$ and $O_\delta(\cdot)$ notations to suppress multiplicative factors that depend only on $\delta$.} to distinguish between the two cases. On the other hand, consider a triple $\cT := (i,j,k)$ such $i < j < k$. Then the win probabilities for the the arms playing $\cT$ are $(1/2,1/4,1/4)$ if $i = 1$, and $(1/3,1/3,1/3)$ otherwise, and it would take only $O_\delta(1)$-queries to distinguish between the two, which is significantly smaller than that of the dueling feedback setting.  
\end{rem}

 

\section{$(r,\teta,\eta)$-\br: Algorithm and Analysis  for the general \br\, Choice Model under Noise}
\label{sec:nbr}

Interestingly, our findings show that even for the noisy settings, Algorithm \ref{alg:ordbr} is a correct $(\epsilon,\delta)$-PAC algorithm when the correlation matrix nearly has a $0$-$1$-block diagonal structure (Thm. \ref{thm:ub_noisy_ordbr}). Formally, our main results are stated as Thm. \ref{thm:noisy_bbox} and \ref{lem:noisy_prep} which gives the precise dependence on noise-vs-the suboptimality gap and its trade-off with learning rate.

\vspace{-5pt}
\subsection{At most $\teta$-MI: Analysis for nearly-independent \irum\, model}

In this section we discuss the setting of noisy-block rank model such that items across the blocks are at most ``$\teta$-identical'' (precisely at most $\teta$-mutual information):

\textbf{$\teta$-\irum:} This is a generalization of \irum\, model where the noise distributions $\cD_i$'s are no longer independent, but can have at most $\teta$-mutual information, i.e.  Corr$(\zeta_i,\zeta_j) \le \teta$, for any pair of distinct arms $i,j \in [n]$ (for any $\teta \in [0,1)$). Clearly, setting $\teta = 0$, we recover the original \irum \, models, as studied in \cite{SG18,SG20,AzariRB+14}.

The main result of this subsection is to show that under `low noise' $(\tilde{\eta})$, our algorithm \brpb\, (Alg. \ref{alg:ordbr}) still finds an $\epsilon$-best item with $O(r \epsilon^{-2} \log(n/\delta))$ sample-complexity. Specifically, the important aspect we note is while the guarantees of Seq-PB sub-routine of Alg. \ref{alg:ordbr} rely on the independence structure across blocks, it can be shown to yield correct results even under $\teta$-\irum (Thm. \ref{thm:noisy_bbox}) model under a separation of scores assumption (see Thm. \ref{thm:noisy_bbox}). Before stating Thm. \ref{thm:noisy_bbox}, we find it useful to introduce some definitions:

\begin{defn}[Best-Item-Advantage-Ratio]
\label{def:rat}
Given any \irum\, model, and subset $S \subseteq [n]$, the advantage ratio of the best item of set $S$, $i^*_S:=\argmax_{i \in S}\mu_i$, over any other item $j \in S \setminus \{i^*_S\}$ is defined as \text{\ratio}$(j,S)$:  
\[
\text{BAR}_{\otimes_{\ell \in [n]} \cD_{\ell}}(\cI)(j,S) = \frac{Pr(i^*_S|S)}{Pr(j|S)}.
\] 
(Explicit use of the subscript ${\otimes_{\ell \in [n]} \cD_{\ell}}(\cI)$ represents the underlying \irum \, model.)
\end{defn}

\begin{cor}
 \label{cor:barlb}
 It is easy to note that by definition, BAR$(j,S) > \frac{1}{k Pr(j|S)}$ for any subset $S$ of size $k$, since the win probability of the best item $i_S^*$ in $S$ is at least $\frac{1}{k}$. 
 \end{cor}

\begin{defn}[Minimum Best-Item-Advantage-Ratio]
\label{def:mrat}
The $\epsilon$-\ratio, (\mrat), is defined to be the minimum (worst case) \ratio\, an $\epsilon$-best item gets against an non-$\epsilon$ best item $(j)$ globally, irrespective of which set it appears inside. More precisely, let $[n]_\epsilon:=\{i \in [n] \mid \mu_i > \mu_1 - \epsilon\}$ denotes the set of all $\epsilon$-best items in $[n]$, then for any \irum\, model $\cI$, we define its \mrat$(\cI)$\, to be: 
\begin{align}
\label{eq:mrat}
\nonumber & \text{\mrat}_{\otimes_{\ell \in [n]} \cD_{\ell}}(\cI) =\\
& \min_{S \in \{S \mid S \cap [n]_{\epsilon} \neq \emptyset\}, j \in S \setminus [n]_\epsilon}\text{BAR}_{\otimes_{\ell \in [n]} \cD_{\ell}}(\cI)(j,S).
\end{align}
\end{defn}

Note \mrat\, is a measure of \emph{worst-case quality separation} of an $\epsilon$-best item over a non $\epsilon$-best item (in terms of item-preferences), which is the key complexity factor in the sample complexity analysis of Seq-PB subroutine used in Alg. \ref{alg:ordbr} as stated below:

\begin{restatable}[Correctness and Sample Complexity of Seq-PB on $\teta$-\irum]{thm}{noisybbox}
	\label{thm:noisy_bbox}
Consider any subsetwise preference model \irum\, $\cI$ on the underlying noise distribution $\cD$, such that \mrat$(\cI) \ge 1 + \frac{4c\epsilon}{1-2c}$ for some $\cD$-dependent constant $c = c(\cD) > 0$. Then Seq-PB$(n,k,\epsilon,\delta,c)$ is an $(\epsilon,\delta)$-{PAC} algorithm on any instance $\nu$ of $\teta$-\irum\, model with sample complexity $O(\frac{n}{c^2\epsilon^2} \log \frac{k}{\delta})$, 
for any $\teta < [0,\frac{c^2\epsilon^2}{32^2 k^4}\big)$. 
Here $\nu$ being an instance of the \irum\, model corresponding to the noise distribution $\cD$.
\end{restatable}

\vspace{-15pt}
\begin{proof}[Proof sketch of Thm. \ref{thm:noisy_bbox}]
The proof depends on the following main lemma which ensures if the \mrat\, of any \irum\, based preference model is bounded away by a certain threshold, then the \mrat\, of the corresponding $\teta$-\irum\, model will also be bounded away by nearly a same threshold for `small' $\teta$.

\begin{restatable}[Lower bound for the advantage ratio $\teta$-\irum\, model]{lem}{thmrat}
\label{thm:ratio}
Consider \irum\, model of Thm. \ref{thm:noisy_bbox}. Then for any $\teta$-\irum\, based subsetwise preference model, say $\cI'$, we have
\begin{equation*}
\label{eq:mat_z}
\text{\mrat}_{\cD}(\cI') \ge 1 + \frac{2c\epsilon}{1-2c}. 
\end{equation*}
\end{restatable}

Given the above lemma, the rest of the argument follows same as Thm. $4$ of \cite{SG20} as it pivots on the main assumptions on \mrat$(\cI)$ as achieved in Lem. \ref{thm:ratio}. The complete proof of Lem. \ref{thm:ratio} and Thm. \ref{thm:noisy_bbox} is given in Appendix \ref{app:xbb}.
\end{proof}

\subsection{At least $(1-\eta)$-Correlation: Nearly-identical intra-block noise}

In this subsection we discuss the setting of noisy intra block items when items inside the block are almost identical, with at least $(1-\eta)$-correlation. Our main finding here is to show that the pre-processing step of \brpb\, (Alg. \ref{alg:ordbr}), which exploits the intra-block item-correlations, is robust under `mild-noise' $\eta$ or more precisely when they are at least $(1-\eta)$-correlated. 

\begin{restatable}[]{thm}{noisyprep}
	\label{lem:noisy_prep}
	Let $\eta \in \big[0, \min\{\frac{1}{192},\min_{j \neq 1}\Delta^4_{1j}/16\}\big]$\footnote{Here $\Delta_{1j} := \mu_1 - \mu_j$ denotes the suboptimality gap between items $j$ and $1$}. Then with probability at least $1 - \delta/2$, the pre-processing step (Lines $6$-$10$) constructs a set $S$ of size at most $r$ such that (i) $1 \in S$ and (ii) $|S \cap \cB_a| \leq 1$ for every $a \in [r]$. Furthermore, the number of samples queried in the pre-processing step is at most $O(n^3 \log n/\delta)$.
\end{restatable}

We defer the proof of the above to Appendix \ref{app:xpp}.

\begin{rem}
\label{rem:nbr_corr}
 Thm. \ref{lem:noisy_prep} says that we need $\eta$ to precisely depend on the suboptimality gaps, $\Delta_{1j}$, of the items residing in the block of the best item-$1$. Note if $\epsilon < \min_{j \in [n]\sm \{1\}}\Delta_{1,j}$, then this immediately implies $\eta < \epsilon^4/16$ is sufficient enough for Thm. \ref{lem:noisy_prep} to hold good. But if $\epsilon > \min_{j \in [n]\sm\{1\}}\Delta_{1,j}$, then we explicitly need $\eta \leq \min_{j \neq 1}\Delta^4_{1j}/16$ in order to be left with only $r$ items at the end of the pre-processing step of Alg. \ref{alg:ordbr}.
 \end{rem} 

\subsection{Performance of \brpb\, (Alg. \ref{alg:ordbr}) for $(r,\teta,\eta)$-\br \, model}

Combining Thm. \ref{thm:noisy_bbox} and \ref{lem:noisy_prep}, the main result follows:

\begin{restatable}[]{thm}{ubnbr}
\label{thm:ub_noisy_ordbr}
Consider any $(r,\teta,\eta)$-\br\, subset choice model \rumb\, with noise distribution $\cD$, such that $\teta < [0,\frac{c^2 \epsilon^2}{32^2 k^4}\big)$ and $\eta^{1/4} \leq \min\big(\epsilon,\mu_1-\max_{i \in [n] \setminus \{1\}}\mu_i\big)/2 $. Then Alg. \ref{alg:ordbr} is $(\epsilon,\delta)$-{PAC} item with sample complexity $O(\frac{r}{c\epsilon^2} \log \frac{n}{\delta})$, $c = c(\cD)$ being a noise distribution dependent constant. 
\end{restatable}


\vspace{-10pt}
\begin{proof}
The result immediately follows combining the claims for the correctness and sample complexity of the pre-processing step (Lem. \ref{lem:noisy_prep}), and the subsequent analysis of running the Seq-PB blackbox (Thm. \ref{thm:noisy_bbox}), for the noisy-\br\, setup.
\end{proof}



\section{Conclusion and Future Work}
\label{sec:conclusion}
\vspace{-10pt}
In this work we explore the role of correlations structure in the $\epsilon$-best item learning problem in preference bandits which is motivated from the search of faster learning rates with subsetwise preferences compared to the dueling feedback (pairwise preference). 
Our result shows that playing sets of larger size (i.e, $\geq 3$) can allow learners to exploit the underlying correlation structure better in comparison to playing sets of size $2$. We also show that our results holds even when the correlation structure has low block rank in an approximate sense.

\textbf{Future Works.} This work opens a suite of interesting directions for future investigation to study the influence of item correlations in preference bandits, specially due to the absence of works along this line. In particular, note that the preprocessing step of our algorithm incurs a $\tilde{O}(n^3)$-sample complexity which is prohibitive with $n$ large. This naturally motivates the question of designing faster pre-processing routines and to understand sample complexity lower bounds for the same. From a broader perspective, additional directions could be to explicitly model item features/attributes to induce correlation along item utilities, study other classes of low rank structures, or even define a general notion of item correlations directly in terms of the preference relations. Another interesting problem would be to understand the role of graphical feedback or side information \cite{yish15} in learning from preferences. Finally, it would be interesting to explore analogous notions of correlation in settings with infinite arms.



\bibliographystyle{plainnat}
\bibliography{bib_rankbattle,chrombandits}

\newpage
\onecolumn
\allowdisplaybreaks

\appendix
{
\section*{\centering \Large{Supplementary for \papertitle}}
}



\section{\textbf{Proof of Lemma \ref{THM:LB_GEN}}}			\label{app:gen-rank}

\lbgen*

\begin{proof}
Let $\epsilon \in (0,1)$ be a small constant to be fixed later. Define the score vectors $\bmu = (\mu,\ldots,\mu)$ and $\bmu^\epsilon = (\mu + \epsilon,\mu, \cdots, \mu)$. Furthermore, we define the correlation matrix $\Sigma$ in terms of its Cholesky decomposition $\Sigma:= {\bf V}{\bf V}^{\top}$ where ${\bf V} = ({\bf v}_1,\ldots,{\bf v}_k)^\top \in \R^{k \times 2}$. Since the diagonal entries of $\Sigma$ are ones, the corresponding ${\bf v}_i$'s are unit vectors, and therefore, we can write ${\bf v}_i = {\bf u}(\alpha_i)$, where ${\bf u}(\alpha)$ is the unit vector $(\cos \alpha, \sin \alpha)$. We define the corresponding $\alpha_i$'s as follows.
\[
\alpha_i = 
\begin{cases}
0 & \mbox{ if } i = 1, \\
\pi & \mbox{ if } i = k, \\
\pi/4 & \mbox{if } i \notin \{0,k\}, i \mbox{ is even } \\
-\pi/4 & \mbox{ if } i \notin \{0,k\}, i \mbox{ is odd}
\end{cases}
\]
To begin with, we shall first analyze the win probabilities with respect to the uniform score vector $\bmu = (\mu,\ldots,\mu)$. In that case, it easy to verify that 
\begin{align}			\label{eqn:win-prob}
\Pr_{\bmu}\left(i | [k]\right) &= \Pr_{{\bf g} \sim N({\bf 0},{\bf I}_{2 \times 2})} \left(\argmax_{i \in [k]} \langle {\bf g}, {\bf v}_i \rangle = i \right) 
= \Pr_{\alpha \sim [0,2\pi]} \left(\argmax_{i \in [k]} \langle {\bf u}(\alpha), {\bf v}_i \rangle = i \right)
\end{align}
Here the first equality holds since all the scores are identical, and the second equality holds since the (i) the event inside the probability expression is scale invariant and (ii) the Gaussian measure is ``rotation invariant''. The RHS of the above equation implies that the win probabilities are determined using the angular measure of the sectors $S_i := \{\alpha | \langle {\bf u}(\alpha), v_i \rangle > \langle {\bf u}(\alpha),v_j \rangle \forall \ j \neq i \}$. Using this observation, the win probabilities are easily computed -- we summarize them below.
\begin{equation}				\label{eqn:cases}
	\Pr_{\bmu}\left(i | [k]\right) = 
	\begin{cases}
		\frac18 & \mbox{if } i = 1, \\
		\frac38 & \mbox{ if } i = k, \\
		\frac{1}{8(k-2)}  & \mbox{ otherwise. }
	\end{cases}
\end{equation}
Now, define the function $f:[0,1] \to [0,1]$ corresponding to the mapping
\[
f(\epsilon) \overset{\rm def}{=} \Pr_{\bf g}\left(\argmax_{i \in [k]} \langle {\bf g},{\bf v}_i \rangle + \mu^\epsilon_i = k \right)
- \Pr_{\bf g}\left(\argmax_{i \in [k]} \langle {\bf g},{\bf v}_i \rangle + \mu^\epsilon_i = 1 \right),
\]
i.e, in words, the above measures the difference in the win probabilities of arms $k$ and $1$ when the subset is played with score vector $\boldsymbol{\mu}^\epsilon$. In particular, note that by definition, $f(0)$ is the difference between the win probabilities of arms $k$ and $1$ with respect to the score vector $\bmu^0 = \bmu$, which is $1/4$ from \eqref{eqn:cases}. Furthermore, since $f$ is a continuous function of $\epsilon$, there exists a choices of $\epsilon_0$ (possibly depending on parameters $\mu$ and $k$) such for every $\epsilon \leq \epsilon_0$ we have $f(\epsilon) \geq f(0) - 1/8 \geq 1/8$. Therefore, using the definition of $f$, for every such small enough choice of $\epsilon$ we get that  
\[
\Pr_{\bf g}\left(\argmax_{i} \langle {\bf g},{\bf v}_i \rangle + \mu^\epsilon_i = k \right)
- \Pr_{\bf g}\left(\argmax_{i} \langle {\bf g},{\bf v}_i \rangle + \mu^\epsilon_i = 1 \right) 
= f(\epsilon) \geq \frac18,
\]
which implies that the win probability of arm $k$ is larger than that of arm $1$ by $1/8$ when the subset $[k]$ is played. Since arm is the unique $\epsilon/2$-best arm with respect to the perturbed score vector ${\bf \mu}^\epsilon$, this establishes the desired claim.
\end{proof}

\section{Proofs for Section \ref{sec:br_ub}}
\label{app:br_ub}

\subsection{Proof of Theorem \ref{THM:UB_ORDBR}}

\ubordbr*

\begin{proof}
The key observation used in the proof of the theorem is the following lemma which gives high probability guarantees on the structure of the set $S$.

\preproc*

We defer the proof of Lem. \ref{lem:pre-proc} for now, and use it to complete the proof of Theorem \ref{thm:ub_ordbr}. Suppose the guarantees of the above lemma hold for $S$. Then, since $|S \cap \cB_i| \leq 1$ for every $i \in [r]$, the corresponding arms are independent. Therefore, instantiating Theorem 1 of \cite{SG20} with $S,\epsilon,\delta/2$, we get that with probability at least $1 - \delta/2$, the Algorithm 1 from \cite{SG20} returns a $\epsilon$-best arm of $S$ with probability at least $1 - \delta/2$ (see Theorem 4~\cite{SG20}). Furthermore, since $1 \in S$, any $\epsilon$-best arm of $S$ would also be an $\epsilon$-best arm in $[n]$.

Therefore combining this with the guarantee of Lemma \ref{lem:pre-proc}, we get that with probability at least $1 - \delta$, Algorithm \ref{alg:ordbr} returns an $\epsilon$-best arm. All that remains is bound the sample complexity. The for loop involves $O(n^2\log(nr/\delta))$-pulls. From Theorem 4 of \cite{SG20} we know that the winner determination step requires $O(r\epsilon^{-2} \log (r/\delta))$-pulls. Since the second term dominates as $\epsilon \to 0$, we get that the overall sample complexity is bounded by $O(r\epsilon^{-2} \log (r/\delta))$.
\end{proof}

\subsection{Technical Lemmas for Thm. \ref{thm:ub_ordbr}}

\subsubsection{Proof of Lemma \ref{lem:pre-proc}}

\preproc*

\begin{proof}
The proof of Lemma \ref{lem:pre-proc} is established using a couple of straightforward claims which we state and prove below.

\wprob*

\begin{proof}[Proof of Claim $1$.]
	Recall that $X_a := \mu_a + \zeta_a$ are the variables corresponding to the reward for arms $a = 1,i,j$. If $i,j \in \cB_1$, both inequalities follow trivially. Now we consider three cases.
	
	{\bf Case (i)}: Suppose $i \in \cB_1,j \notin \cB_1$ (the other case can be argued identically). Then $\Pr(i |\cT) = 0$ and $\Pr(1|\cT) \geq \Pr(j|\cT)$ (since $\mu_1> \mu_i,\mu_j$) and hence $\Pr(1|\cT) \geq 1/2$.
	
	{\bf Case (ii)}: Suppose $i,j \notin \cB_1$ and $i$, $j$ belong to distinct blocks. Then, $1,i,j$ are independent, and since $\mu_1 \geq \mu_i,\mu_j$ it follows that $\Pr(1|\cT)\geq\Pr(i|\cT),\Pr(j|\cT)$ and hence $\Pr(1|\cT) \geq 1/3$.
	
	{\bf Case (iii)} Suppose $i,j \notin \cB_1$ and $i$, $j$ belong to the same block. Without loss of generality, assume $\mu_i \geq \mu_j$. Since $\zeta_i = \zeta_j$, then this implies that $X_i \geq X_j$ with probability $1$. Hence,  
	Therefore, 
	\[
	\Pr(1|\cT) = \Pr_{X_1,X_i} \left(X_1 > X_i \right) \geq \Pr_{\zeta_1,\zeta_j}\left(\zeta_1 > \zeta_j\right) \geq \frac12
	\]
	where the last step follows due to $\zeta_1$ and $\zeta_j$ being identical and independent random variables. 
\end{proof}

\wprobb*
\begin{proof}[Proof of Claim $2$.]
	Note that the setting of the claim is identical to that of case(iii) from the proof of Claim \ref{cl:win-prob1}, and therefore, we have $\Pr(\{i,j\}|\cT) \leq 1/2$. Furthermore, since $X_i \geq X_j$ with probability $1$, we have $\Pr(i| \cT) \geq \Pr(j| \cT)$. Hence, 
	\[
	\frac12 \geq \Pr\left(\{i,j\}|\cT \right) = \Pr\left(i | \cT\right) + \Pr\left(j | \cT\right) \geq 2\Pr\left(j | \cT\right),
	\]
	which on rearranging gives us the claim.
\end{proof}

Using the above, we establish the following lemma which gives w.h.p. characterization of the set of items marked during the for loop.

\flag*

We defer the proof of the above lemma to Section \ref{sec:flag} and use the conclusions and finish the proof of the current lemma by assuming the conclusions of Lemma \ref{lem:flag}. First consider any block $\cB_i$ with $i \in [r]$. If there exists a unique item $j_i \in \argmax_{j' \in \cB_i} \mu_{j'}$ with the largest score, then using Lemma \ref{lem:flag}, for every $j' \in \cB_i \setminus \{j\}$ we have ${\rm Flag}(j') = 1$, and ${\sf Flag}(j_i) = 0$. On the other hand, if more than one item have the largest bias in $\cB_i$, then for every $j \in \cB_i$ we have ${\rm Flag}(j) = 1$. In other words, for every $\cB_i$, we must have at most one element of $j_i \in \cB_i$ for which ${\rm Flag}(j) = 0$. Finally, since the optimal arm i.e, arm $1$ has bias strictly larger than every other arm, including the arms of $\cB_1 \setminus\{1\}$, by the above argument we must have ${\rm Flag}(1) = 0$.
\end{proof}

\subsubsection{Proof of Lem. \ref{lem:flag}}		\label{sec:flag}

\flag*

\begin{proof}
	For arguing the first part, fix items $j,j' \in [n]$ such that $\mu_j \geq \mu_{j'}$ and they belong to the same block. Now consider the triple $\cT:= (1,j,j')$. From Claim \ref{cl:win-prob2} it follows that $\Pr(j'|\cT) \leq 1/4$ and hence using Hoeffding's inequality we get
	\begin{equation}				\label{eq:prob1}
		\Pr\Big({\rm Flag}(j') = 0\Big) \leq \Pr\Big(N_\cT(j') > 0.26t \Big) \leq \Pr\Big(N_\cT(j') - \E N_{\cT}(j') > t/100 \Big) \leq \exp(-10^4t^2/2) \leq \delta/4n^2
	\end{equation}
	where the last inequality holds due to our choice of $t := 2 \times 10^{4}\log(4n^2/\delta)$. On the other hand, consider any $\cT$ such that $1 \in \cT$. Then from Claim \ref{cl:win-prob1} we know that $\Pr(1|\cT) \geq 1/3$ and therefore, again using Hoeffding's inequality we get that
	\begin{equation}				\label{eq:prob2}		
	\Pr\Big(N_\cT(1) \leq 0.26 t\Big) \leq \Pr\Big(N_{\cT}(1) - \E N_\cT(i) < - t/100 \Big)
	\leq \exp(-(10)^4t^2/2) \leq \delta/4n^2
	\end{equation}
	Therefore, taking a union bound over at most $(n - 1)$ events corresponding to \eqref{eq:prob1} and ${n \choose 2}$ events corresponding to \eqref{eq:prob2}, we have that with probability at least $1 - (n + n^2)(\delta/4n^2) \geq 1 - \delta/2$, the conclusions of the lemma hold simultaneously. 
\end{proof}

\section{Proofs for Section \ref{sec:br_lb}}
\label{app:br_lb}

\subsection{Pseudocode: Reducing \texttt{I-RUM}$(r,k)$ into instances of \texttt{BR-RUM}$(n,k,r)$}
\label{app:pc2}

\begin{center}
	\begin{algorithm}[H]
		\caption{Algorithm $\cA^{I-RUM}$} 
		\label{alg:lb}
		\begin{algorithmic}[1]
			\STATE {\bfseries Input:} 
			\STATE ~~~ Set of items: $[n]$, Subset size: $2 \leq  k \leq n$.
			\STATE ~~~ Error bias: $\epsilon >0$, Confidence parameter: $\delta >0$.
			\STATE {\bfseries Initialize:} 
			\STATE ~~~ Run $\cA^{\rm C-RUM}$ by simulating the lifted distribution described in \eqref{eqn:redn} as follows.  
			\FOR {Iterations $t = 1,2,\ldots$} 
			\STATE ~~~ Let $S_t$ be the subset queried by $\cA^{\rm C-RUM}$ in iteration $t$.
			\IF {$S_t \cap [r] \neq \emptyset$}
			\STATE ~~~ Play subset $S_t \cap [r]$ and feed the corresponding winner to $\cA^{C-RUM}$.
			\ELSE 
			\STATE ~~~ Feed a uniformly random element $i_t \sim S_t$ to $\cA^{\rm C-RUM}$.
			\ENDIF 
			\STATE ~~~ If $\cA^{\rm C-RUM}$ returns a item, break.
			\ENDFOR
			\STATE {\bfseries Output:} Item returned by $\cA^{\rm C-RUM}$.
		\end{algorithmic}
	\end{algorithm}
	\vspace{-2pt}
\end{center}

\subsection{Proof of Theorem \ref{THM:LB_ORDBR}}

\lbordbr*

\begin{proof}
The proof of the lower bound proceeds via a reduction from the problem of $(\epsilon,\delta)$-PAC learning the best item in \texttt{I-RUM}$(r,k)$ instance a $(\epsilon,\delta)$-PAC learning problem in a \texttt{BR-RUM}$(n,k,r)$ instance. Formally let $\cI := (\bmu,\cD)$-be a \texttt{I-RUM}$(r,k)$ instance. Then we construct a \texttt{BR-RUM}$(n,k,r)$-instance $\cI':= (\cD,\bmu',\Sigma)$ as follows.
\begin{equation}					\label{eqn:redn}
\mu'_i = 
\begin{cases}
\mu_i & \mbox{ if } i \in  [r], \\
-\infty& \mbox{otherwise}.
\end{cases}
\qquad\qquad\qquad\qquad
\Sigma := 
\begin{bmatrix}
{\rm Id}_{r-1} & {\bf 0}_{(r-1) \times (n - r + 1)} \\ 
{\bf 0}_{(n - r + 1) \times (r-1)} & {\bf 1}_{(n - r + 1) \times (n - r+1)}  
\end{bmatrix}
\end{equation}
In the above, we use ${\bf 0}_{a \times b}$ denote the all zeros matrix with $a$-rows and $b$-columns, and similarly, ${\bf 1}_{a \times b}$. Note that $\Sigma$ as constructed above has block rank $r$ and hence $\cI'$ is indeed a \texttt{BR-RUM}$(n,k,r)$ instance. Now given an $(\epsilon,\delta)$-PAC Algorithm $\cA^{\rm C-RUM}$ for \texttt{BR-RUM}$(n,k,r)$-instances, we construct an algorithm $\cA^{I-RUM}$ as shown in Alg. \ref{alg:lb}.

{\bf Correctness of Reduction}. Towards establishing the correctness of the reduction, as a first step, we claim that for any iteration $t$, Algorithm \ref{alg:lb} correctly simulates the feedback model corresponding to instance $\cI'$. Formally, this is equivalent to showing that 
\[
\Pr\Big(\textnormal{Alg \ref{alg:lb} sends $i$ from $S_t$}\Big) = \Pr_{\bmu',\Sigma}(i | S_t), \qquad \forall ~ i \in S_t,
\]
where $S_t \subseteq [n]$ is the subset of size at most $k$ played by the inner algorithm $\cA^{\rm C-RUM}$ in iteration $t$. We argue this by considering two cases.  If $S_t \cap [r] = \emptyset$, then for any $i \in S_t$,
\begin{equation}				\label{eqn:equiv-1}
\Pr\Big(\textnormal{Alg \ref{alg:lb} sends $i$ from $S_t$}\Big) = \frac{1}{|S_t|} = \Pr_{\bmu',\Sigma}(i | S_t),
\end{equation}
where the first equality holds since the algorithm returns a uniformly random element in $S_t$ and the second inequality holds since all the items in $S_t$ have identical scores. On the other hand, suppose $\tilde{S}_t := S_t \cap [r] \neq \emptyset$. Then for any $i \in \tilde{S}_t$, we have 
\begin{equation}				\label{eqn:equiv-2}
	\Pr\Big(\textnormal{Alg \ref{alg:lb} sends $i$ from $S_t$}\Big) = \Pr\left(i = \argmax_{i' \in S_t \cap [r]} X_{i'}\right)
	= \Pr\left( i = \argmax_{i' \in S_t } X_{i'} \right) = \Pr_{\bmu',\Sigma}(i|S_t), 
\end{equation}
where the first equality is due to Line $9$ of Algorithm \ref{alg:lb}, the second equality uses the fact that for any $i' \in S_t \setminus \tilde{S}_t$ we have $\mu'_{i'} = -\infty$ and hence $X_{i'} < X_j$ for every $j \in \tilde{S}_t$ almost surely. Finally, the last equality follows using the definition of $\Pr(\cdot|S_t)$. The same identity (as \eqref{eqn:equiv-2}) also holds for every $j \in S_t \setminus \tilde{S}_t$ using identical arguments i.e., \eqref{eqn:equiv-2} holds for every $i \in S_t$. 

The above arguments taken together imply that for every iteration $t$, the feedback received by Algorithm $\cA^{\rm C-RUM}$ matches the feedback model of the instance $\cI'$ . Therefore, using the $(\epsilon,\delta)$-PAC guarantee of $\cA^{\rm C-RUM}$ on \texttt{BR-RUM}$(n,k,r)$, it follows that Algorithm \ref{alg:lb} returns a $\epsilon$-best item with respect to score vector $\bmu'$ with probability at least $1- \delta$. Finally, note that since $\mu'_j = -\infty$ for every $j \in [n] \setminus [r]$, the set of $\epsilon$-best arms with respect to score vector $\bmu'$ is identical to the set of $\epsilon$-best arms on score vector $\bmu$ and hence, Algorithm \ref{alg:lb} actually returns an $\epsilon$-best arm with respect to score vector $\bmu$ i.e., it is $(\epsilon,\delta)$-PAC on instance \texttt{I-RUM}$(r,k)$. Finally, since Theorem \ref{thm:lb_ind} implies that the sample complexity of any  $(\epsilon,\delta)$-algorithm for \texttt{I-RUM} instances on $r$ arms (with any $k \geq 2$) is $\Omega(r\epsilon^{-2}\log(1/\delta))$, it follows that Algorithm \ref{alg:lb} must have sample complexity at least $\Omega(r\epsilon^{-2} \log(1/\delta))$.

\end{proof}

\subsection{Proof of Theorem \ref{THM:LB_DUEL}}

\lbduel*

\begin{proof}
Same as the proof of Thm. \ref{thm:lb_ind}, the arguments is based on the change of measure based lemma stated as Lem. \ref{lem:gar16}. We constructed the following specific instances for our purpose and assume $\cD$ to be the $\cN(0,1)$ noise for this case. Also since the learner is supposed to play subsets of size only $k=2$, we denote the action (arm) set in this case by $\cA:=\{\{i,j\}\subseteq [n] \mid i <j\}$ (note that, for the purpose of deriving the lower bound we can safely exclude repeated arm-pairs $(i,i)$ from $S$ as playing such a duel reveals no preference information, for the same reason the KL divergences for such sets are also going to be $0$ while we would be using Lem. \ref{lem:gar16}).

Let $\bnu^1$ be the true distribution associated with the bandit arms, given by the utility parameters:
\begin{align*}
\textbf{True Instance} ~(\bnu^1): \mu_j^1 = \mu, \forall j \in [n]\setminus \{1\}, \text{ and } \mu_1^1 = \mu + \epsilon,
\end{align*}

for some $\mu \in \R_+, ~\epsilon > 0$. Now for every suboptimal item $a \in [n]\setminus \{1\}$, consider the modified instances $\bnu^a$ such that:
\begin{align*}
\textbf{Instance--a} ~(\bnu^a): \mu^a_j = \mu, \forall j \in [n]\setminus \{a,1\}, \, \mu_1^a = \mu, \text{ and } \mu_a^a = \mu+\epsilon.
\end{align*}

For problem instance $\bnu^a, \, a \in [n]\setminus\{1\}$, the probability distribution associated with arm $S \in \cA$ is given by
\[
\nu^a_S \sim Categorical(p_1, p_2, \ldots, p_k), \text{ where } p_i = Pr(i|S), ~~\forall i \in [k], \, \forall S \in \cA,
\]
where $Pr(i|S)$ is as defined in Section \ref{sec:prelims}. 
Note that the only $\epsilon$-optimal arm for \textbf{Instance-a} is arm $a$. 
Further, we assume a size $r$ \br \, structure over $[n]$ arms in every instance $\bnu^a, \, a \in [n]$, such that for $\bnu^a$ set $\cB_1^a = \{a\}$, and the rest of the $(r-1)$ blocks are equally divided in the arms $[n]\sm \{a\}$, such that each of the remaining blocks $\cB_2,\ldots, \cB_r$ gets exactly $\frac{n-1}{r-1}$ arms (rounded to nearest interests such that $\sum_{i = 2}^{r}|\cB_i| = n-1$).

Now applying Lemma \ref{lem:gar16}, for some event $\cE \in \cF_\tau$ we get,

\begin{align}
\label{eq:FI_a}
\sum_{\{S \in \cA : a \in S\}}\E_{\bnu^1}[N_S(\tau_A)]KL(\bnu^1_S, \bnu^a_S) \ge {kl(Pr_{\nu}(\cE),Pr_{\nu'}(\cE))}.
\end{align}

The above result holds from the straightforward observation that for any arm $S \in \cA$, if $\{1,a\}\cap S = \emptyset$, $\bnu^1_S$ is same as $\bnu^a_S$, hence $KL(\bnu^1_S, \bnu^a_S)=0$, $\forall S \in \cA, \,a \notin S$. 
For notational convenience, we will henceforth denote $S^a = \{S \in \cA \mid \{1,a\}\cap S \neq \emptyset\}$. 

Now let us analyze the right-hand side of \eqref{eq:FI_a}, for any pair (duel) $S = (j,j') \in S^a$. 

\textbf{Case 1} ($1 \notin S, a \in S$):  For simplicity first consider the case $1 \notin S$. 
Note that: $\nu^1_S(j) = \nu^1_S(j') = 0.5$.

On the other hand, for problem \textbf{Instance-a}, we have that: $\nu^1_S(i) = 0.5 + \alpha$ if $i = a$ (where we use the result from Lem. \ref{lem:pref_normal}, here $\alpha = \Phi(\epsilon)$), and $\nu^1_S(i) = 0.5 - \alpha$, otherwise.

Now using the following upper bound on $KL(\p_1,\p_2) \le \sum_{z \in Z}\frac{p_1^2(z)}{p_2(z)} -1$, $\p_1$ and $\p_2$ be two probability mass functions on the discrete random variable $Z$ \citep{klub16} we get:

\begin{align*}
KL(\bnu^1_S, \bnu^a_S) & \le (\frac{1}{2^2})\frac{2}{1+2\alpha} + (\frac{1}{2^2})\frac{2}{1-2\alpha} - 1\\
& = \frac{2\alpha}{2}\Big( \frac{1}{1-2\alpha} - \frac{1}{1+2\alpha} \Big) = \alpha\Big( \frac{4 \alpha}{1-4\alpha^2}\Big) = 4 \alpha^2 (1-(2\alpha)^{2})^{-1} \le 8 \alpha^2, 
\end{align*}
where the last inequality follows for any $\epsilon \le \frac{1}{4}$, noting that by definition $\alpha = \Phi(\epsilon) \le \frac{\epsilon}{\sqrt{2 \pi}}$.

\textbf{Case 2} ($1 \in S, a \in S$): Note in this case $S = \{1,a\}$. Here  we get: $\nu^1_S(1) = 0.5 + \alpha, \nu^1_S(a) = 0.5 - \alpha$.
And on the other hand, for problem \textbf{Instance-a}, we have that: $\nu^1_S(a) = 0.5 + \alpha$ and $\nu^1_S(1) = 0.5 - \alpha$, otherwise.

Now using the following upper bound on $KL(\p_1,\p_2) \le \sum_{z \in Z}\frac{p_1^2(z)}{p_2(z)} -1$, $\p_1$ and $\p_2$ be two probability mass functions on the discrete random variable $Z$ \citep{klub16} we get:

\begin{align*}
KL(\bnu^1_S, \bnu^a_S) & \le (\frac{(1+2\alpha)^2}{2^2})\frac{2}{1-2\alpha} + (\frac{(1-2\alpha)^2}{2^2})\frac{2}{1+2\alpha} - 1\\
& = \frac{1}{2}\Big( \frac{(1+2\alpha)^3 + (1-2\alpha)^3 }{1-4\alpha^2}\Big) - 1 = \frac{2(1+ 12\alpha^2)}{2(1-4\alpha^2)} - 1 = 16 \alpha^2 (1-(2\alpha)^{2})^{-1} \le 32 \alpha^2, 
\end{align*}
where again the last inequality follows for any $\epsilon \le \frac{1}{4}$, and since $\alpha = \Phi(\epsilon) \le \frac{\epsilon}{\sqrt{2 \pi}}$.

\textbf{Case 3} ($1 \in S, a \notin S$): Finally in this case $S = \{1,i\}$ for some $i \neq a$. Here we get: $\nu^1_S(i) = 0.5 + \alpha, \nu^1_S(a) = 0.5 - \alpha$ and, for problem \textbf{Instance-a}: $\nu^1_S(a) = \nu^1_S(i) = 0.5$. Here again it it can be proved that $KL(\bnu^1_S, \bnu^a_S) \le 16 \alpha^2$.

Now note that the only $\epsilon$-optimal arm for any \textbf{Instance-a} is arm $a$, for all $a \in [n]$.
Now, consider $\cE_0 \in \cF_\tau$ be an event such that the algorithm $A$ returns the element $i = 1$, and let us analyze the left-hand side of \eqref{eq:FI_a} for $\cE = \cE_0$. Clearly, $A$ being an $(\epsilon,\delta)$-PAC algorithm, we have $Pr_{\bnu^1}(\cE_0) > 1-\delta$, and $Pr_{\bnu^a}(\cE_0) < \delta$, for any sub-optimal arm $a \in [n]\setminus\{1\}$. Then we have 

\begin{align}
\label{eq:win_lb2a}
kl(Pr_{\bnu^1}(\cE_0),Pr_{\bnu^a}(\cE_0)) \ge kl(1-\delta,\delta) \ge \ln \frac{1}{2.4\delta}
\end{align}

where the last inequality follows from \cite{Kaufmann+16_OnComplexity} (Eqn. $(3)$).

Now applying \eqref{eq:FI_a} for each modified bandit \textbf{Instance-$\bnu^a$}, and summing over all suboptimal items $a \in [n]\setminus \{1\}$ we get,

\begin{align}
\label{eq:win_lb2.5a}
\sum_{a = 2}^{n}\sum_{\{S \in \cA \mid a \in S\}}\E_{\bnu^1}[N_S(\tau_A)]KL(\bnu^1_S,\bnu^a_S) \ge (n-1)\ln \frac{1}{2.4\delta}.
\end{align}

Moreover, using above derived bounds in the KL terms of the form $KL(\bnu^1_S, \bnu^a_S)$, the term of the right-hand side of \eqref{eq:win_lb2.5a} can be further upper bounded as

\begin{align}
\label{eq:win_lb3a}
\sum_{a = 2}^{n}&\sum_{\{S \in \cA \mid \{a,1\} \cap S\neq \empty\}} \E_{\bnu^1}[N_S(\tau_A)]KL(\bnu^1_S,\bnu^a_S) \le \sum_{S \in \cA}\E_{\bnu^1}[N_S(\tau_A)]2\Big( 32\epsilon^2 \Big).
\end{align}

Finally noting that $\E_{\bnu^1}[\tau_A] = \sum_{S \in \cA}[N_S(\tau_A)]$, combining \eqref{eq:win_lb3a} and \eqref{eq:win_lb2.5a}, we get 

\begin{align*}
(64\epsilon^2)\E_{\bnu^1}[\tau_A] =  \sum_{S \in \cA}\E_{\bnu^1}[N_S(\tau_A)](63\epsilon^2) \ge (n-1)\ln \frac{1}{2.4\delta}.
\end{align*}
Thus above construction shows the existence of a problem instance $\bnu = \bnu^1$, such that $\E_{\bnu^1}[\tau_A] = \Omega(\frac{n}{\epsilon^2}\ln \frac{1}{2.4\delta})$, which concludes the proof.
\end{proof}

\subsection{Technical Lemmas for Thm. \ref{THM:LB_DUEL}}

\begin{lem}					
\label{lem:pref_normal}
Consider $X_1 = \mu_1 + \zeta_1$ and $X_2 = \mu_2 + \zeta_2$, where $\zeta_1, \zeta_2 \overset{\text{iid}}{\sim} \cN(0,1)$. Then
	\[
	\Pr\left(X_1 > X_2\right) = \frac12 + \Phi\Big( \frac{\mu_1 - \mu_2}{\sqrt 2}\Big),
	\]
	where $\Phi: \R \mapsto \R$ is such that $\Phi(x) = \int_{0}^{x} \frac{1}{\sqrt{2 \pi}}e^{-y^2/2}  dy, ~\forall x \in \R$.
\end{lem}

\begin{proof}
Let $\phi(\cdot)$ denotes the pdf of standard normal distribution $\cN(0,1)$, i.e for any $x \in \R$, $\phi(x) = \frac{1}{\sqrt{2 \pi}}e^{-x^2/2}$.

	Then by definition we can write
	\begin{align*}
	\Pr(X_1 > X_2) = 
	\Pr\left(\zeta_2 - \zeta_1 < \mu_1 - \mu_2 \right)
	&= \Pr\left( \frac{\zeta_2 - \zeta_1}{\sqrt 2} < \frac{\mu_1 - \mu_2}{\sqrt 2} \right) \overset{(a)}{=} \int_{-\infty}^{\frac{\mu_1 - \mu_2}{\sqrt 2}} \phi (x ) dx  \\
	& = \int_{-\infty}^{0} \phi (x )  dx + \int_{0}^{\frac{\mu_1 - \mu_2}{\sqrt 2}} \phi (x )  dx = \ 0.5 + \Phi\Big( \frac{\mu_1 - \mu_2}{\sqrt 2} \Big)  
	\end{align*}
	where $(a)$ follows noting that since $\zeta_1$ and $\zeta_2$ are independent standard normal random variables, $ \frac{\zeta_2 - \zeta_1}{\sqrt 2}$ also follows $\cN(0,1)$.
\end{proof}

\subsection{Sample Complexity Lower Bound For Independent RUM with variable-sized subsetwise plays}
\label{app:lb_ind}
\begin{restatable}[Sample Complexity Lower Bound for \indrum]{thm}{lbind}
	\label{thm:lb_ind}
	Given $\epsilon \in (0,1/4]$, $\delta \in (0,1]$, $r,k \in [n]$, for any $(\epsilon,\delta)$-PAC algorithm for \prob\, problem, there exists an instance of \texttt{BR-RUM}$(n,k,n,\bmu)$, say $\nu$ (i.e. an \indrum\, instance with $r = n$), where the expected sample complexity of $A$ on $\nu$ is  at least $\Omega\big( \frac{n}{\epsilon^2} \ln \frac{1}{2.4\delta}\big)$.
\end{restatable}

\begin{proof}
	Our result is similar to the spirit of \cite{SGwin18}, however their setup considers subsets of fixed size $k$ and we assumed the learner has the flexibility to play any subsets $S \subseteq [n]$ of length $|S| = 1,2,\ldots,k$. Due to this additional flexibility in the feedback model (compared to \cite{SGwin18}), their lower bound does not imply a fundamental performance limit for our case, and we need to derive the claim of \ref{thm:lb_ind} independently. 
	
	Before proving the above lower bound result we recall the main lemma from \citep{Kaufmann+16_OnComplexity} which is a general result for proving information theoretic lower bound for bandit problems:
	
	Consider a multi-armed bandit (MAB) problem with $n$ arms or actions $\cA = [n]$. At round $t$, let $A_t$ and $Z_t$ denote the arm played and the observation (reward) received, respectively. Let $\cF_t = \sigma(A_1,Z_1,\ldots,A_t,Z_t)$ be the sigma algebra generated by the trajectory of a sequential bandit algorithm up to round $t$.
	\begin{restatable}[Lemma $1$, \citep{Kaufmann+16_OnComplexity}]{lem}{gar16}
		\label{lem:gar16}
		Let $\nu$ and $\nu'$ be two bandit models (assignments of reward distributions to arms), such that $\nu_i ~(\text{resp.} \,\nu'_i)$ is the reward distribution of any arm $i \in \cA$ under bandit model $\nu ~(\text{resp.} \,\nu')$, and such that for all such arms $i$, $\nu_i$ and $\nu'_i$ are mutually absolutely continuous. Then for any almost-surely finite stopping time $\tau$ with respect to $(\cF_t)_t$,
		\vspace*{-5pt}
		\begin{align*}
			\sum_{i = 1}^{n}\E_{\nu}[N_i(\tau)]KL(\nu_i,\nu_i') \ge \sup_{\cE \in \cF_\tau} kl(Pr_{\nu}(\cE),Pr_{\nu'}(\cE)),
		\end{align*}
		where $kl(x, y) := x \log(\frac{x}{y}) + (1-x) \log(\frac{1-x}{1-y})$ is the binary relative entropy, $N_i(\tau)$ denotes the number of times arm $i$ is played in $\tau$ rounds, and $Pr_{\nu}(\cE)$ and $Pr_{\nu'}(\cE)$ denote the probability of any event $\cE \in \cF_{\tau}$ under bandit models $\nu$ and $\nu'$, respectively.
	\end{restatable}
	
	We now proceed to prove our lower bound result of Thm. \ref{thm:lb_ind}.
	
	In order to apply the change of measure based lemma (Lem. \ref{lem:gar16}), we constructed the following specific instances for our purpose and assume $\cD$ to be the \gn$(0,1)$ noise. Also since the learner is supposed to play subsets of size up to $k$, we denote the action (arm) set in this case by $\cA:=\{S\subseteq [n] \mid |S| \in [k]\}$.
	
	\begin{align*}
		\text{True Instance} ~(\bnu^1): \mu_j^1 = 1-\epsilon, \forall j \in [n]\setminus \{1\}, \text{ and } \mu_1^1 = 1,
	\end{align*}
	
	Note the only $\epsilon$-optimal arm in the true instance is arm $1$. Now for every sub-optimal item $a \in [n]\setminus \{1\}$, consider the modified instances $\bnu^a$ such that:
	\begin{align*}
		\text{Instance--a} ~(\bnu^a): \mu^a_j = 1-2\epsilon, \forall j \in [n]\setminus \{a,1\}, \, \mu_1^a = 1-\epsilon, \text{ and } \mu_a^a = 1.
	\end{align*}
	
	For any problem instance $\bnu^a, \, a \in [n]\setminus\{1\}$, the probability distribution associated with arm $S \in \cA$ is given by
	\[
	\nu^a_S \sim Categorical(p_1, p_2, \ldots, p_k), \text{ where } p_i = Pr(i|S), ~~\forall i \in [k], \, \forall S \in \cA,
	\]
	where $Pr(i|S)$ is as defined in Section \ref{sec:prelims}. 
	Note that the only $\epsilon$-optimal arm for \textbf{Instance-a} is arm $a$. Now applying Lemma \ref{lem:gar16}, for any event $\cE \in \cF_\tau$ we get,
	
	\begin{align}
		\label{eq:FI_aa}
		\sum_{\{S \in \cA : a \in S\}}\E_{\bnu^1}[N_S(\tau_A)]KL(\bnu^1_S, \bnu^a_S) \ge {kl(Pr_{\nu}(\cE),Pr_{\nu'}(\cE))}.
	\end{align}
	
	The above result holds from the straightforward observation that for any arm $S \in \cA$, $|S| \in [k], $ with $a \notin S$, $\bnu^1_S$ is same as $\bnu^a_S$, hence $KL(\bnu^1_S, \bnu^a_S)=0$, $\forall S \in \cA, \,a \notin S$. 
	For notational convenience, we will henceforth denote $S^a = \{S \in \cA : a \in S\}$. 
	
	Now let us analyze the right-hand side of \eqref{eq:FI_a}, for any set $S \in S^a$. 
	
	\textbf{Case-1:} First let us consider $S \in S^a$ such that $1 \notin S$. Note that in this case:
	\begin{align*}
		\nu^1_S(i) = \frac{1}{|S|}, \text{ for all } i \in S
	\end{align*}
	
	On the other hand, for problem \textbf{Instance-a}, we have that: 
	
	\begin{align*}
		\nu^a_S(i) = 
		\begin{cases} 
			\frac{e^1}{(|S|-1)e^{1-2\epsilon} + e^1} \text{ when } S(i) = a,\\
			\frac{e^{1-2\epsilon}}{(|S|-1)e^{1-2\epsilon} + e^1}, \text{ otherwise.}
		\end{cases}
	\end{align*}
	
	Again using the upper bound on $KL(\p_1,\p_2) \le \sum_{z \in Z}\frac{p_1^2(z)}{p_2(z)} -1$ for probability mass functions $\p_1$ and $\p_2$~\citep{klub16} we get:
	
	\begin{align*}
		KL(\bnu^1_S, \bnu^a_S) & \le (|S|-1)\frac{(|S|-1)e^{1-2\epsilon} + e^1}{|S|^2(e^{1-2\epsilon})} + \frac{(|S|-1)e^{1-2\epsilon} + e^1}{|S|^2e^1}-1\\
		& = \frac{(|S|-1)}{|S|^2}\bigg( e^\epsilon - e^{-\epsilon}\bigg)^2 = \frac{(|S|-1)}{|S|^2}e^{-2\epsilon}(e^\epsilon-1)^2 \le \frac{8\epsilon^2}{|S|} \text{ for any } \epsilon \in \bigg[0,\frac{1}{2}\bigg]
	\end{align*}
	
	\textbf{Case-2:} Now let us consider the remaining set in $S^a$ such that $S \owns 1,a$. Similar to the earlier case in this case we get that:
	
	\begin{align*}
		\nu^a_S(i) = 
		\begin{cases} 
			\frac{e^1}{(|S|-1)e^{1-\epsilon} + e^1} \text{ when } S(i) = 1,\\
			\frac{e^{1-\epsilon}}{(|S|-1)e^{1-\epsilon} + e^1}, \text{ otherwise.}
		\end{cases}
	\end{align*}
	
	On the other hand, for problem \textbf{Instance-a}, we have that: 
	
	\begin{align*}
		\nu^a_S(i) = 
		\begin{cases} 
			\frac{e^{1-\epsilon}}{(|S|-2)e^{1-2\epsilon} + e^{1-\epsilon} + e^1} \text{ when } S(i) = 1,\\
			\frac{e^{1}}{(|S|-2)e^{1-2\epsilon} + e^{1-\epsilon} + e^1} \text{ when } S(i) = a,\\
			\frac{e^{1-2\epsilon}}{(|S|-2)e^{1-2\epsilon} + e^{1-\epsilon} + e^1}, \text{ otherwise}
		\end{cases}
	\end{align*}
	
	Now using the previously mentioned upper bound on the KL divergence, followed by some elementary calculations one can show that for any $\big[0,\frac{1}{4} \big]$:
	\vspace{-10pt}
	\begin{align*}
		KL(\bnu^1_S, \bnu^a_S) & \le \frac{8\epsilon^2}{|S|}
	\end{align*}
	\vspace{-10pt}
	
	Thus combining the above two cases we can conclude that for any $S \in S^a$, $KL(\bnu^1_S, \bnu^a_S) \le \frac{8\epsilon^2}{|S|}$, and as argued above for any $S \notin S^a$, $KL(\bnu^1_S, \bnu^a_S) = 0$.
	
	Note that the only $\epsilon$-optimal arm for any \textbf{Instance-a} is arm $a$, for all $a \in [n]$.
	Now, consider $\cE_0 \in \cF_\tau$ be an event such that the algorithm $A$ returns the element $i = 1$, and let us analyze the left-hand side of \eqref{eq:FI_a} for $\cE = \cE_0$. Clearly, $A$ being an $(\epsilon,\delta)$-PAC algorithm, we have $Pr_{\bnu^1}(\cE_0) > 1-\delta$, and $Pr_{\bnu^a}(\cE_0) < \delta$, for any sub-optimal arm $a \in [n]\setminus\{1\}$. Then we have 
	
	\begin{align}
		\label{eq:win_lb2}
		kl(Pr_{\bnu^1}(\cE_0),Pr_{\bnu^a}(\cE_0)) \ge kl(1-\delta,\delta) \ge \ln \frac{1}{2.4\delta}
	\end{align}
	
	where the last inequality follows from \citep{Kaufmann+16_OnComplexity} (Eqn. $3$).
	
	Now applying \eqref{eq:FI_aa} for each modified bandit \textbf{Instance-$\bnu^a$}, and summing over all suboptimal items $a \in [n]\setminus \{1\}$ we get,
	
	\begin{align}
		\label{eq:win_lb2.5}
		\sum_{a = 2}^{n}\sum_{\{S \in \cA \mid a \in S\}}\E_{\bnu^1}[N_S(\tau_A)]KL(\bnu^1_S,\bnu^a_S) \ge (n-1)\ln \frac{1}{2.4\delta}.
	\end{align}
	
	Using the upper bounds on $KL(\bnu^1_S,\bnu^a_S)$ as shown above, the right-hand side of \eqref{eq:win_lb2.5} can be further upper bounded as:
	
	\begin{align}
		\label{eq:win_lb3}
		\nonumber \sum_{a = 2}^{n}&\sum_{\{S \in \cA \mid a \in S\}} \E_{\bnu^1}[N_S(\tau_A)]KL(\bnu^1_S,\bnu^a_S) \le \sum_{S \in \cA}\E_{\bnu^1}[N_S(\tau_A)]\sum_{\{a \in S \mid a \neq 1\}}\frac{8\epsilon^2}{|S|}\\
		& = \sum_{S \in \cA}\E_{\bnu^1}[N_S(\tau_A)]{|S| - \big(\1(1 \in S)\big)}\frac{8\epsilon^2}{|S|} \le \sum_{S \in \cA}\E_{\bnu^1}[N_S(\tau_A)]{8\epsilon^2}.
	\end{align}
	
	Finally noting that $\E_{\bnu^1}[\tau_A] = \sum_{S \in \cA}[N_S(\tau_A)]$, combining \eqref{eq:win_lb2.5} and \eqref{eq:win_lb3}, we get 
	
	\begin{align}
		\label{eq:win_lb_fin}
		(8\epsilon^2)\E_{\bnu^1}[\tau_A] =  \sum_{S \in \cA}\E_{\bnu^1}[N_S(\tau_A)](8\epsilon^2) \ge (n-1)\ln \frac{1}{2.4\delta}.
	\end{align}
	Thus rewriting Eqn. \ref{eq:win_lb_fin} we get $\E_{\bnu^1}[\tau_A]  \ge \frac{(n-1)}{8\epsilon^2}\ln \frac{1}{2.4\delta}$. 
	The above construction shows the existence of a problem instance of \indrum\, with $n$ items (\texttt{BR-RUM}$(n,k,n,\bmu)$ model) where any $(\epsilon,\delta)$-PAC algorithm requires at least $\Omega(\frac{n}{\epsilon^2}\ln \frac{1}{2.4\delta})$ samples.
\end{proof}

\section{\textbf{Infeasibility in \br~Choice Models}} 

\begin{restatable}[Problem Infeasibility for Subsetwise-Queries]{lem}{imposfxd}
	\label{thm:impos_fxd}
	For the problem of \prob\, with the restriction of playable subsets of only fixed size $k$, it is possible to construct problem instances of \rumb, such that $\exists$ $i,j \in [n]$ such that $\mu_i > \mu_j + \epsilon$ but $P(i|S) < P(j|S), \forall S \subseteq [n]$ for some choice of $\epsilon \in (0,1)$.
\end{restatable}

\begin{proof}
	Consider a problem instance \textbf{Instance $\cI$: } Consider a simple problem instance with any general $n\ge 10$, $r = 3$, $k = n/2 \ge 5$, and for the purpose of this specific instances assume $\cD$ is just \gn$(0,1)$ noise.
	
	Let the block structure be $\cB_1 = \{1\}$, $\cB_2 = \{2\}$ and $\cB_{3} = \{3,\ldots,n\}$. And let $\mu_1 = mu+c\epsilon$, for any $c \to 1_+$ is the score of the best-item $i^* = 1$. We set $\mu_2 = \mu$, and $\mu_i = \mu+\epsilon, ~\forall i \in [n]\sm [2]$. So the items in the third block are nearly as good as the best item $1$ as $c \to 1_+$.
	
	However, any $k$-sized subset $S$ such that $S$ containing Item-$2$ should have at least $(k-2) \ge 3$ items from $\cB_3$ if $1 \in S$ as well, or all $(k-1)$ items from $\cB_3$ along with Item-$2$. This implies $P(i|S) = 
	\begin{cases}
	O(1/3(k-2)) ~~\text{when } 1 \in S\\
	O(1/2(k-1)) ~~\text{when } 1 \notin S
	\end{cases}, ~\forall i \in \cB_3\cap S$. 
	In either case, clearly $P(i|S) = O(1/k)$ for any $i \in \cB_3\cap S$. Where as $P(2|S) = O(1)$ only (precisely $P(2|S) \approx \frac{1}{3} - \epsilon$ when $1 \in S$, and $P(2|S) \approx \frac{1}{2} - \epsilon$ when $1 \notin S$). Noting $k \ge 3$ can be arbitrarily large and also $\epsilon \in (0,1)$ can also be arbitrarily small, this proves the claim. 
\end{proof}

\section{Appendix for Section \ref{sec:nbr}}
\label{app:nbr}



\subsection{Proof of Thm. \ref{thm:noisy_bbox} }
\label{app:xbb}

\textbf{Notation. } Define $\Delta_{ij}:= \mu_i-\mu_j$, for any item pair $(i,j) \in [n]\times[n]$. {For simplicity, we denote $\teta = \eta$.}

\noisybbox*

\begin{proof}
The proof crucially depends on the following main lemma which ensures if the \mrat\, of any \irum\, based preference model is bounded away by a certain threshold, then the \mrat\, of the corresponding $\teta$-\irum\, model has to be bounded away by nearly the same threshold as long as $\eta$ is not too large. The formal claim is as stated below:

\thmrat*

Given the above lemma, the rest of the argument follows same as the proof steps shown for Thm. $4$ of \cite{SG20} as it pivots on the main assumptions on \mrat$(\cI)$ as achieved in Lem. \ref{thm:ratio}. We summarize the key steps for the completeness.

\begin{enumerate}
 \item Given Lem. \ref{thm:ratio}, following the same line of argument as shown in Lem. $9$ of \cite{SG20} we get that, upon \emph{Rank-Breaking}, the effective-pairwise probability $p_{ij|S}:= \frac{Pr_{\cD}(i|S)}{Pr_{\cD}(ij|S)}$ ($Pr_{\cD}(ij|S):= Pr_{\cD}(i|S)+Pr_{\cD}(j|S)$ denotes the probability of either item $i$ or $j$ being the winner of set $S$), of any $\epsilon$-best item $i \in [n]_\epsilon$ winning over an non-$\epsilon$ best item $j \notin [n]_\epsilon$ is still bounded away from $\frac{1}{2}$ by $O(\epsilon)$ margin. Precisely for any such (near-best,suboptimal) item pair $(i,j)$, we have $p_{ij|S} > \frac{1}{2} +\frac{c\Delta_{ij}}{2(1-2c)}$ as long as $\Delta_{ij} > \frac{\epsilon}{4}$, irrespective of the underlying set $S$.
 
 \item Now given the fact that, for any $S$ and any (near-best,suboptimal) item pair $(i,j)$, we have $p_{ij|S} > \frac{1}{2} +\frac{c\Delta_{ij}}{2(1-2c)}$, we can simply replicate the proof of Lem. 11 of \cite{SG20} to argue that for any such subset $S$, Seq-PB$(n,k,\epsilon,\delta,c)$ (see description in Alg. \ref{alg:ordbr} or Alg $1$ of \cite{SG20}) would retain a near best item of $S$, say $i_S$ such that $\mu_{i_s}> \mu_{i_S^*} - {\epsilon_\ell}/c$, after $O(\frac{k}{\epsilon_\ell^2}\log \frac{k}{\delta_\ell})$ $k$-subsetwise queries (with high probability $(1-\delta_\ell)$) for any $\epsilon_\ell,\delta_\ell \in (0,1]$.
 
 \item Finally, combining the above two claims and proceeding similar to the proof of Thm $4$ \cite{SG20}, the correctness and total sample complexity of Seq-PB$(n,k,\epsilon,\delta,c)$ follows for any $\teta$-\irum\, model.
 \end{enumerate}
 \end{proof}

\vspace{-10pt}

\subsubsection{Proof of Lem. \ref{thm:ratio}}

\thmrat*

\begin{proof}
Following the definition of \mrat\, recall that explicitly:
 
\begin{align}
\label{eqn:def-mrat}
\text{\mrat}_{\cD}(\cI) = \min_{S \in \{S \mid S \cap [n]_{\epsilon} \neq \emptyset\}, j \in S \setminus [n]_\epsilon}\frac{Pr_{\cD}\big(\{X_{i_S^*} > \max(X_{\{-i_S^*\}}^{ S})\}\big)}{Pr_{\cD}\big(\{X_j > \max(X_{\{-j\}}^{ S})\}\big)}
\end{align}
where for any $i \in [n], S \subseteq [n]$, denote $X_{\{-i\}}^{ S} = \{\cup_{i \in S}X_j\}\sm \{X_i\}$, and suppose $(S^*,j^*)$ is the minimizer set of the right-hand side expression of \mrat\, above.

Now for any subset $S,j$, using the `Cross-block Approximate Independence'-property of $\teta$-\br\, model and Lem. \ref{lem:cross-pb1}, we get:
\begin{align}
\label{eq:mrat1}
\frac{Pr_{\cD}\big(\{X_{i_S^*} > \max(X_{\{-i_S^*\}}^{ S})\}\big)}{Pr_{\cD}\big(\{X_j > \max(X_{\{-j\}}^{S})\}\big)}
>  \frac{Pr_{\otimes_{\ell \in  S} \cD_{\ell}}\big(\{X_{i_S^*} > \max(X_{\{-i_S^*\}}^{ S})\}\big) - k\sqrt \teta}{Pr_{\otimes_{\ell \in S} \cD_{\ell}}\big(\{X_j > \max(X_{\{-j\}}^{ S})\}\big) + k\sqrt \teta}.
\end{align}
Define $\gamma$ as 
\begin{align}				\label{eqn:fac-prob}
\gamma := \frac{Pr_{\otimes_{\ell \in  S} \cD_{\ell}}\big(\{X_{i_S^*} > \max(X_{\{-i_S^*\}}^{ S})\}\big)}{Pr_{\otimes_{\ell \in S} \cD_{\ell}}\big(\{X_j > \max(X_{\{-j\}}^{ S})\}\big)}
\end{align}
For brevity, denote $p = Pr_{\otimes_{\ell \in  S} \cD_{\ell}}\big(\{X_j > \max(X_{\{-j\}}^{S})\}\big)$. From Cor. \ref{cor:barlb} we have $\gamma p \geq 1/k$. We consider two cases.

{\bf Case (i)}: Suppose $\gamma \geq 4$. Then we proceed to bound \eqref{eq:mrat1} as 
\begin{align}
\frac{Pr_{\otimes_{\ell \in  S} \cD_{\ell}}\big(\{X_{i_S^*} > \max(X_{\{-i_S^*\}}^{ S})\}\big) - k\sqrt \teta}{Pr_{\otimes_{\ell \in S} \cD_{\ell}}\big(\{X_j > \max(X_{\{-j\}}^{ S})\}\big) + k\sqrt \teta} 
& \overset{1}{=}\frac{\gamma p - k\sqrt{\teta}}{p + k\sqrt{\teta}} 				 \nonumber\\
& = \left(\frac{\gamma p - k\sqrt{\teta}}{p + k\sqrt{\teta}} - 2\right) + 2 	\nonumber\\
& = \frac{(\gamma - 2)p - 2k\sqrt{\teta}}{p + k\sqrt{\teta}}  + 2 \nonumber\\
& \overset{2}{\geq} \frac{\gamma p/2  - 2k\sqrt{\teta}}{p + k\sqrt{\teta}}  + 2 		\nonumber\\
& \overset{3}{\geq} \frac{1/(2k)  - 2k\sqrt{\teta}}{p + k\sqrt{\teta}}  + 2 			\nonumber\\
& \overset{4}{\geq}  2 				\label{eqn:mrat-bd1}
\end{align} 
where step $1$ follows from the definition of $\gamma$ and $p$, step $2$ follows from the assumption $\gamma \leq 4$ for this case, step $3$ uses $\gamma p \geq 1/k$ and step $4$ follows from $\teta \leq 1/(16k^4)$.

{\bf Case (ii)}: Suppose $\gamma \leq 4$. Then using the definition of $\gamma$ from \eqref{eqn:fac-prob} this implies, 
\begin{align}					\label{eqn:frac-lb}
Pr_{\otimes_{\ell \in S} \cD_{\ell}}\big(\{X_j > \max(X_{\{-j\}}^{ S})\}\big) \geq \frac14\cdot Pr_{\otimes_{\ell \in S} \cD_{\ell}}\big(\{X_{i_S^*} > \max(X_{\{-i_S^*\}}^{ S})\}\big) \geq \frac{1}{4k} \geq k\sqrt{\teta}
\end{align}
where the last inequality follows from our choice of $\teta \le \frac{1}{16k^4}$. Now recall (from Thm. \ref{thm:noisy_bbox}), we are given that \mrat$_{\otimes_{\ell \in S} \cD_{\ell}}(\cI) - 1 > \tilde c \epsilon$, where $\tilde c := \frac{4c}{(1-2c)}$, which implies:
\begin{align*}
\min_{S' \in \{S \mid S \cap [n]_{\epsilon} \neq \emptyset\}, j' \in S \setminus [n]_\epsilon} \frac{Pr_{\otimes_{\ell \in S'} \cD_{\ell}}\big(\{X_{i^*_{S'}} > \max(X_{\{-i^*_{S'}\}}^{S'})\}\big)}{Pr_{\otimes_{\ell \in S'} \cD_{\ell}}\big(\{X_{j'} > \max(X_{\{-j'\}}^{S'})\}\big)} - 1>  \tilde c \epsilon
\end{align*}
Assume the minimum above is attained for the pair $(\tS,\tj)$ Then continuing from \eqref{eq:mrat1} we get:
\begin{align}
\label{eq:mrat2}
\nonumber & \frac{Pr_{\cD}\big(\{X_{i_S^*} > \max(X_{\{-{i_S^*}\}}^{S})\}\big)}{Pr_{\cD}\big(\{X_j > \max(X_{\{-j\}}^{S})\}\big)} -1 >\frac{Pr_{\otimes_{\ell \in  S} \cD_{\ell}}\big(\{X_{i_S^*} > \max(X_{\{-{i_S^*}\}}^{S})\}\big) - k \sqrt \teta}{Pr_{\otimes_{\ell \in  S} \cD_{\ell}}\big(\{X_j > \max(X_{\{-j\}}^{S})\}\big) + k \sqrt \teta} -1 \\
\nonumber &> \frac{Pr_{\otimes_{\ell \in S} \cD_{\ell}}\big(\{X_{i_S^*} > \max(X_{\{-i_S^*\}}^{S})\}\big) - Pr_{\otimes_{\ell \in S} \cD_{\ell}}\big(\{X_{ j} > \max(X_{\{- j\}}^{ S})\}\big)-2k\sqrt{\teta}}{Pr_{\otimes_{\ell \in S} \cD_{\ell}}\big(\{X_{ j} > \max(X_{\{- j\}}^{ S})\}\big) + k\sqrt \teta}\\
\nonumber &> \frac{Pr_{\otimes_{\ell \in \tS} \cD_{\ell}}\big(\{X_{i_{\tS}^*} > \max(X_{\{-i_{\tS}^*\}}^{\tS})\}\big) - Pr_{\otimes_{\ell \in\tS} \cD_{\ell}}\big(\{X_{\tj} > \max(X_{\{-\tj\}}^{\tS})\}\big)-2k\sqrt{\teta}}{Pr_{\otimes_{\ell \in\tS} \cD_{\ell}}\big(\{X_{\tj} > \max(X_{\{-\tj\}}^{\tS})\}\big) + k\sqrt \teta}\\
& > \frac{\tilde c \epsilon }{1 + k\sqrt \teta/Pr_{\otimes_{\ell \in\tS} \cD_{\ell}}\big(\{X_{\tj} > \max(X_{\{-\tj\}})\}\big)} - 8k^2\sqrt{\teta} > \frac{\tilde c \epsilon}{2},				
\end{align} 
where the second last inequality uses $Pr_{\otimes_{\ell \in\tS} \cD_{\ell}}\big(\{X_{\tj} > \max(X_{\{-\tj\}}^{\tS})\}\big) \geq 1/4k$, the last inequality follows from \eqref{eqn:frac-lb} and the fact that $\eta \leq \frac{c^2\epsilon^2}{32^2 k^4}$. 

 Combining the two cases i.e, \eqref{eqn:mrat-bd1} and \eqref{eq:mrat2}, for any subset $S$ of size $k$ we have 
\[
\frac{Pr_{\cD}\big(\{X_i > \max(X_{\{-i\}}^{S})\}\big)}{Pr_{\cD}\big(\{X_j > \max(X_{\{-j\}}^{S})\}\big)}
\geq \min\left\{1 + \frac{\tilde c \epsilon}{2}, 2\right\}
\] 
which completes the proof. Since the above holds for any subset of size $k$, it also holds for the minimizer in \eqref{eqn:def-mrat} $(S^*,j^*)$, and hence the claim follows.
\end{proof}

\subsection{Proof of Thm. \ref{lem:noisy_prep}}
\label{app:xpp}

\noisyprep*

\begin{proof}
	We first establish the first part of the lemma. To begin with, let $\cF_0$ denote the set of triples $\cT \in {n\choose 3} $ such that $1 \in S$ and let $\cF_1$ denote the set of triples of the form $(1,i,j)$ such that $i,j$ belong to the same block. For any triple $\cT$, let $m(\cT)$ denote the arm in $\cT$ with the minimum win probability with respect to triple $\cT$ i.e.,  
	\[
	m(\cT) := \argmin_{i \in \cT} \Pr\left(i | \cT\right).
	\]
	Now for any triple $\cT$ we observe that (a) if  $\cT \in \cF_0$, using Lemma \ref{lem:cross-pb1} we have $\Pr(1|\cT) \geq 1/3 - 4\sqrt{\eta}$ (b) if $\cT \in \cF_1$ using Corollary \ref{corr:cross-win} we have $\Pr(m(\cT)|\cT) \leq 1/4 + 4\sqrt{\eta}$. Furthermore, we can show that the bounds (a) and (b) also hold approximately even for the (empirical) win probability estimates. Towards that, define the event $\cE$ as 
	\[
	\cE : = \Big\{\forall \ \cT \in \cF_0 : N_{\cT}(1) \geq 0.32t \Big\} \wedge \Big\{\forall \ \cT \in \cF_1 : N_{\cT}(m(\cT)) \leq 0.26t \Big\} 
	\]
	Then using Hoeffding's inequality and our choice of $t = O(\log (4n^3/\delta))$ (from Line 5 of Algorithm \ref{alg:ordbr}) we can bound the probability of the event $\cE$ not occurring as:
	\begin{align*}
	&\Pr\left(\Big\{\exists \ \cT \in \cF_0 : N_{\cT}(1) < 0.32t \Big\} \vee \Big\{\exists \ \cT \in \cF_1 : N_{\cT}(m(\cT)) > 0.26t \Big\} \right) \\
	&\leq \sum_{\cT \in \cF_0}\Pr\left(N_{\cT}(1) < 0.32t \right) + \sum_{\cT \in \cF_1}\Pr\left( N_{\cT}(m(\cT)) > 0.26t \right) \\
	&\leq \sum_{\cT \in \cF_0}\Pr\left(\frac{N_{\cT}(1)}{t} - \Pr(1|\cT)  < -0.01 \right) 
	+ \sum_{\cT \in \cF_1}\Pr\left( \frac{N_{\cT}(m(\cT))}{t} - \Pr(m(\cT)|\cT) > 0.01 \right) \\
	&\leq 2{n \choose 3} \frac{\delta }{4n^3} \leq \frac{\delta}{2}.
	\end{align*}
	The above implies that event $\cE$ holds with probability at least $1 - \delta/2$. We now argue that conditioned on $\cE$, the subset $S$ constructed in Line 10 of Alg. \ref{alg:ordbr} satisfies the properties (i) and (ii) with probability $1$. To see (i), observe that since for every triple $\cT\in \cF_0$ we have $N_{\cT}(1) \geq 0.32t$, we must have ${\rm Flag}(1) = 0$ and hence $1 \in S$. We now argue property (ii) by contradiction. Suppose property (ii) is violated. Then there exists arms $i,j$ such that $\{i,j\} \subseteq \cB_a \cap S$ (for some $a \in [r]$). Now consider the triple $\cT:= (1,i,j)$ and without loss of generality, assume that $m(\cT) = j$. By construction, we have $\cT\in \cF_1$, and hence, conditioning on the event $\cE$ we have $N_{\cT}(j) \leq 0.26t$ which in turn implies that we must have ${\rm Flag}(j) = 1$ at the end of the for loop, which contradicts the fact that $j \in S$. Hence, conditioned on $\cE$ we must have that $S$ satisfies properties (i) and (ii).
	
	Finally, observe that in lines $6$-$10$, each triple $\cT \in {n\choose 3}$, is played $t = O(\log (n^3/\delta))$ delta times, which implies that the total number of samples queried here is $O(n^3 \log (n/\delta))$.
	
\end{proof}

\subsection{Technical Lemmas for Appendix \ref{app:xbb} and \ref{app:xpp}}

\begin{lem}				
\label{lem:noisy-win}
	Given any triple $\cT:= (1,i,j)$, then we  have $\Pr\left(1 | \cT\right) \geq 1/3 - 4\sqrt{\eta}$ (recall Thm.. \ref{thm:ub_noisy_ordbr} assumes $\mu_1 > \max\{\mu_i,\mu_j\}+ 2\eta^{1/4}$). 
\end{lem}

\begin{proof}
	The proof consists of several cases depending on the block memberships of $i,j$.
	
	{\bf Case (i)}. Suppose $i,j \in \cB_1$. In this case we note that ${\rm Corr}(\zeta_a,\zeta_b) \geq 1 - \eta$ for any $a,b \in \{1,i,j\}$. We claim that `with high probability'
	\begin{equation}					\label{eqn:imply}
	\Big\{\zeta_1 > \max(\zeta_i - \eta^{1/4}, \zeta_j - \eta^{1/4})\Big\} \Rightarrow \Big\{X_1 > \max(X_i,X_j)\Big\}.
	\end{equation}
	To see this observe that if $\zeta_1 \geq \zeta_i - \eta^{1/4}$ then,
	\[
	X_i  = \mu_i + \zeta_i \overset{1}{<} \mu_i + \zeta_1 + \eta^{1/4} \oleq{2} \mu_1 + \zeta_1 - \eta^{1/4} \leq \mu_1 + \zeta_1 = X_1,
	\]
	where step $1$ is using $\zeta_i < \zeta_1 + \eta^{1/4}$ and step $2$ uses the fact that $\mu_1 \geq \mu_i + 2\eta^{1/4}$. Using identical arguments we can also show that $X_1 > X_j$, which establishes \eqref{eqn:imply}. Therefore, using the contrapositive of \eqref{eqn:imply} we get that 
	\begin{align*}
	\Pr_{X_1,X_i,X_j} \left(X_1 \leq \max(X_i,X_j)\right)
	&\leq \Pr_{\zeta_1,\zeta_i,\zeta_j}\left(\zeta_1 \leq \max\left(\zeta_i - \eta^{1/4} , \zeta_j - \eta^{1/4}\right)\right)  \\
	&\leq \Pr_{\zeta_1,\zeta_i}\left(\zeta_1 < \zeta_j - \eta^{1/4}\right) + \Pr_{\zeta,\zeta_j}\left(\zeta_1 < \zeta_j - \eta^{1/4}\right) \\
	&\leq 4\sqrt{\eta}.
	\end{align*}
	where in the last step we use Lemma \ref{lem:corr-conc} on both terms. Therefore, with probability at least $1 - 4 \sqrt{\eta}$ we have $X_1 > \max(X_i,X_j)$.
	
	{\bf Case (ii)} Suppose $i \in \cB_1$ and $j \in \cB_a$ for some $a \neq 1$. Then as in the previous case, since ${\rm Corr}(\zeta_1,\zeta_i) \geq 1- \eta$ 
	we have $\Pr(X_1 > X_i) \geq 1 - 2\sqrt{\eta}$. Furthermore, since $I(\zeta_1;\zeta_j) \leq \eta$ and $\mu_1 \geq \mu_j$, we have 
	\[
	\Pr_{X_1,X_j}\left(X_1 > X_j\right) \geq \Pr_{\zeta_1,\zeta_j} \left(\zeta_1 > \zeta_j\right) \geq \frac12 - \sqrt{\eta},
	\]
	where the last step follows using Lemma \ref{lem:cross-pb1}. Therefore, by union bound we have 
	\[
	\Pr_{X_1,X_i,X_j}\left( X_1 \leq \max\left(X_i,X_j\right)\right)
	\leq \Pr_{X_1,X_i} \left(X_1 \leq X_i\right) + \Pr_{X_1,X_j} \left(X_1 \leq X_j\right) \leq \frac12 + \sqrt{\eta} + 2\sqrt{\eta} \leq \frac12 + 3\sqrt{\eta}.
	\]	
	
	{\bf Case (iii)} Suppose $i,j \in \cB_a$ for some $a \neq 1$. Here we have $I(\zeta_1;\zeta_i), I(\zeta_1;\zeta_j) \leq \eta$ and ${\rm Corr}(\zeta_i,\zeta_j) \geq 1 - \eta$. Without loss of generality, assume that $\mu_i \geq \mu_j$. Since $I(\zeta_1;\zeta_i) \leq \eta$, using Lemma \ref{lem:cross-pb1} we have 
	\begin{equation}					\label{eqn:event1}
	\Pr_{\zeta_1,\zeta_i} \left(\zeta_1 \geq \zeta_i\right) \geq \frac12 - \sqrt{\eta}.
	\end{equation}	
	Furthermore, since ${\rm Corr}(\zeta_i,\zeta_j) \geq 1 - \eta$, using Lemma \ref{lem:corr-conc}
	\begin{equation}					\label{eqn:event2}
	\Pr_{\zeta_i,\zeta_j} \left(\zeta_j \leq \zeta_i + \eta^{1/4}\right) \geq 1 - 2\sqrt{\eta}.
	\end{equation}
	We claim that conditioned on the events from \eqref{eqn:event1} and \eqref{eqn:event2} we have $X_1 > \max(X_i,X_j)$ with probability $1$. Indeed, using the event in \eqref{eqn:event1} and $\mu_1 > \mu_i$ we have $X_1 = \mu_1 + \zeta_1 > \mu_i + \zeta_i = X_i$. Furthermore, 
	\[
	X_j = \mu_j + \zeta_j \oleq{\eqref{eqn:event2}} \mu_i + \zeta_j + \eta^{1/4} \leq \mu_1 + \zeta_i  \overset{\eqref{eqn:event1}}{<} \mu_1 + \zeta_1 = X_1,
	\]
	where the middle inequality again uses $\mu_1 \geq \mu_j + 2\eta^{1/4}$ in our setting. Therefore, combining the above observation with the bounds from \eqref{eqn:event1},\eqref{eqn:event2} we get that
	\[
	\Pr_{X_1,X_i,X_j}\left(X_1 > \max(X_i,X_j)\right) 
	\geq \Pr_{\zeta_1,\zeta_i,\zeta_j}\left(\{\zeta_1 \geq \zeta_i\} \wedge \{\zeta_j \leq \zeta_i + \eta^{1/4}\}\right)
	\geq \frac12 - 3\sqrt{\eta}.
	\]
	
	{\bf Case (iv)} Suppose $i \in \cB_a$ and $j \in \cB_b$ where $a \neq b$ and $a,b \neq 1$ i.e., the arms $1,i,j$ belong to distinct blocks. Then using Lemma \ref{lem:cross-pb1} we have 
	\[
	\Pr_{\zeta_1,\zeta_i,\zeta_j}\left(\zeta_1 \leq \max(\zeta_i,\zeta_j)\right) \geq \frac13 - 4\sqrt{\eta}.
	\]
	Since $\mu_1 \geq \mu_i,\mu_j$, we have 
	\[
	\Pr_{X_1,X_i,X_j} \left(X_1 \geq \max(X_i,X_j)\right) \geq \Pr_{\zeta_1,\zeta_i,\zeta_j}\left(\zeta_1 \leq \max(\zeta_i,\zeta_j)\right) \geq \frac13 - 4\sqrt{\eta}.
	\]
	
\end{proof}

The above follows directly from Case (iii) of the above lemma.

\begin{cor}			
\label{corr:cross-win}
	Given a triple $\cT = (1,i,j)$ where $i,j \in \cB_a$, we have that $\Pr\left(\{i,j\}| \cT\right) = \Pr_{X_1,X_i,X_j}\left(\max(X_i,X_j) > X_1 \right) \leq \frac12 + 4\sqrt{\eta}$  ~(assuming $\mu_1 > \max\{\mu_i,\mu_j\} + 2\eta^{1/4}$).
\end{cor}

\subsection{Technical Lemmas for almost independent probability distributions (at most $\eta$-mutual information)}

In this section, we establish win probability bounds for subsets consisting of arms from distinct blocks. Here we use $\|\cdot\|_{\rm TV}$ to denote the total variation distance between a pair of random variables. Recall that for any pair of random variables $X,Y$ defined over a common probability space $\Omega$, the total variation distance between the distributions of $X$ and $Y$, denoted by $P_X$ and $P_Y$, can be expressed as 
\begin{equation}			
\label{eqn:def-tv}
\|P_X - P_Y\|_{\rm TV} = \sup_{S \subset \Omega} \left|\Pr_{X}(S) - \Pr_{Y}(S)\right|. 
\end{equation}
Furthermore, we will also use the fact that mutual information can be expressed as KL divergence between the joint distribution and product measure i.e., $I(X;Y) = D_{\rm KL}(P_{XY} || P_X \otimes P_Y)$. We begin by proving a simple well known property of total variation distance of product measures.

\begin{cl}				\label{cl:prod-tv}
	For any pair of probability measures $\nu_1,\nu_2$ defined over a common probability space $\cX$, given another measure $\nu_3$ (not necessarily defined over the same space), we have $\|\nu_1 \otimes \nu_3 - \nu_2 \otimes \nu_3\|_{\rm TV} \leq \|\nu_1 - \nu_2 \|_{\rm TV}$.
\end{cl}
\begin{proof}
	Let $\nu_3$ be defined over probability space $\cX'$. Then, using the fact that $\|\cdot\|_{\rm TV}$ is actually the $\ell_1$-distance between the probability measures we have 
	\begin{align*}
		\Big\| \nu_1 \otimes \nu_3 - \nu_2 \otimes \nu_3 \Big\|
		&= \int_{x \in \cX}\int_{x'\in \cX'} \left|\Big(\nu_2 \otimes \nu_3\Big)(x,x') - \Big(\nu_1 \otimes \nu_3\Big)(x,x') \right| dx dx'\\
		&= \int_{x \in \cX}\int_{x'\in \cX'} \left|\nu_1(x) \nu_3(x') - \nu_2(x) \nu_3(x') \right| dx dx'\\
		&\leq \int_{x \in \cX}\int_{x'\in \cX'}\nu_3(x') \left|\nu_1(x) - \nu_2(x) \right| dx dx'\\
		&= \int_{x \in \cX}\left|\nu_1(x) - \nu_2(x)\right| dx \\
		&= \|\nu_1 - \nu_2\|_{\rm TV}. 
	\end{align*}
\end{proof}

Next we prove the main lemma of this section which is useful in relating the win-probability profile of items in a subset when they are played with almost independent noise, to that of the independent noise setting.  

\begin{lem}				
\label{lem:cross-pb1}
	Let $(\zeta_i)_{\i \in [k]}$ be jointly distributed with measure $\nu$. Furthermore, suppose for any pair of disjoint subsets $S_1,S_2 \subset [k]$ we have $I(\zeta_{S_1} ; \zeta_{S_2}) \leq \eta$. Then, for any $i \in [k]$, we have 
	\[
	\Pr_{\nu}\left(\zeta_i > \max_{j \in [k] \setminus \{i\}} \zeta_j \right) \geq \Pr_{\otimes_{\ell \in [j]}\nu_{\ell}}\left(\zeta_i > \max_{j \in [k] \setminus \{i\}} \zeta_j \right) - k\sqrt{\eta}.
	\]
	where $\otimes_{\ell \in [k]} \nu_\ell$ is the product measure corresponding to the marginals $\nu_1,\ldots,\nu_k$. 
\end{lem}
\begin{proof}
	We prove the lemma for $i = 1$. For any $\ell \in \{2,\ldots,k\}$, let $\nu_{\ell,\ldots,k}$ denote the joint distribution on the set of random variables $(\zeta_\ell,\ldots,\zeta_k)$. We begin by observing that we can bound
	\begin{align}
	\label{eqn:diff1}
	& \left|\Pr_{\nu}\left(\zeta_1 > \max_{j \in [k] \setminus \{1\}} \zeta_j \right) - \Pr_{\nu_1 \otimes \nu_{2, \ldots ,k}}\left(\zeta_1 > \max_{j \in [k] \setminus \{1\}} \zeta_j \right) \right|      \\
	\nonumber &\oleq{1} \left\| \nu - \nu_{1} \otimes \nu_{2,\ldots,k}\right\|_{\rm TV} \\
	&\oleq{2} \sqrt{D_{\rm KL}\left( \nu || \nu_{1} \otimes \nu_{2,\ldots,k}\right)} = \sqrt{I(\zeta_1;\zeta_{2,\ldots,k})} \leq \sqrt{\eta},		
	\end{align}
	where inequality $1$ is using the definition of $\|\cdot\|_{\rm TV}$ (see \eqref{eqn:def-tv}) and step $2$ is using Pinsker's inequality. For brevity, for every $\ell \in \{2,..,k\}$, define $\nu_{\leq \ell}:= \nu_2 \otimes \cdots \otimes \nu_{\ell - 1} \otimes \nu_{\ell,\ldots,k}$ where $\nu_{\ell,\ldots,k}$ is the joint distribution on the variables $\zeta_\ell,\ldots,\zeta_k$. Now for a fixed $x$, using identical steps we observe that 
	\begin{align*}
	&\left|\Pr_{\nu_{2,\ldots,k}} \left(\max_{j \in \{2,\ldots,k\}} \zeta_j\leq x\right) - \Pr_{\otimes_{2 \leq \ell \leq k} \nu_{\ell}} \left( \max_{j \in \{2,\ldots,k\}} \zeta_j \leq x\right) \right| \\
	& =\left|\Pr_{\nu_{\leq 1}} \left(\max_{j \in \{2,\ldots,k\}} \zeta_j \leq x\right) - \Pr_{\nu_{\leq k}} \left( \max_{j \in \{2,\ldots,k\}} \zeta_j \leq x\right) \right| \tag{Definition of $\nu_{\leq j}$}
	\\
	& \leq \sum_{2 \leq \ell \leq k - 1} \left|\Pr_{\nu_{\leq \ell - 1}} \left(\max_{j \in \{2,\ldots,k\}} \zeta_j \leq x\right) - \Pr_{\nu_{\leq \ell}}  \left( \max_{j \in \{2,\ldots,k\}} \zeta_j \leq x\right) \right|			\tag{Telescoping Sum}\\
	& \leq \sum_{2 \leq \ell \leq k - 1} \Big\|\nu_{\leq \ell - 1} - \nu_{\leq \ell} \Big\|_{\rm TV} \tag{Definition of $\|\cdot\|_{\rm TV}$}\\
	& = \sum_{2 \leq \ell \leq k - 1} \left\|\left(\bigotimes^{\ell - 2}_j \nu_j \right)  \otimes \nu_{\ell-1,\ldots,k} - \left(\bigotimes^{\ell - 1}_j \nu_j \right)  \otimes \nu_{\ell,\ldots,k} \right\|_{\rm TV} \tag{Definition of $\nu_{\leq \ell - 1},\nu_{\leq \ell}$}\\
	& \leq \sum_{2 \leq \ell \leq k - 1} \left\|\Big(\nu_{\ell-1,\ldots,k}\Big) - \Big(\nu_{\ell - 1}  \otimes \nu_{\ell,\ldots,k}\Big) \right\|_{\rm TV} \tag{Claim \ref{cl:prod-tv}}\\
	& \leq \sum_{2 \leq \ell \leq k - 1} \sqrt{D_{\rm KL}\left(\nu_{\ell-1,\ldots,k} || \nu_{\ell - 1}  \otimes \nu_{\ell,\ldots,k}\right)} \tag{Pinsker's Inequality} \\
	& =  \sum_{2 \leq \ell \leq k - 1} \sqrt{I\left(\zeta_{\ell - 1,\ldots,\zeta_k} ; \zeta_{\ell - 1} \otimes \zeta_{\ell,\ldots,k}\right)} \tag{Defn. of $I(\cdot;\cdot)$}\\
	&\leq (k-1) \sqrt{\eta}. 
	\end{align*}
	where in last step we use the bound $I(\zeta_{S_1};\zeta_{S_2}) \leq \eta$ for any pair of disjoint subsets $S_1,S_2 \subset [k]$ in our setting. Using the above estimate, we have
	\begin{align*}
	\Pr_{\nu_1 \otimes \nu_{2,\ldots,k}}\left(\zeta_1 > \max_{j \in [k] \setminus \{1\}} \zeta_j \right)
	& = \int^\infty_{- \infty} f_{\nu_1}(\zeta_1) \Pr_{\nu_{2\ldots,k}} \left(\zeta_1 > \max_{j \geq 2} \zeta_j \right) d{\zeta_1}\\
	& \geq  \int^\infty_{- \infty} f_{\nu_1}(\zeta_1) \Pr_{\otimes_{\ell \geq 2}\nu_{\ell}} \left(\zeta_1 > \max_{j \geq 2} \zeta_j < \zeta_1\right) d{\zeta_1} - (k-1)\sqrt{\eta} \\
	& =  \Pr_{\otimes_{\ell \in [k]} \nu_{\ell} } \Big(\zeta_1 \geq \max\left(\zeta_2,\zeta_3\right)\Big) - (k-1)\sqrt{\eta} .
	\end{align*}
	Therefore, plugging in the above bound into \eqref{eqn:diff1} we get that 
	\begin{align*}
	\Pr_{\nu}\left(\zeta_1 > \max_{j \in [k] \setminus \{1\}} \zeta_j \right)
	&\geq \Pr_{\nu_1 \otimes \nu_{2,\ldots,k}}\left(\zeta_1 > \max_{j \in [k] \setminus \{1\}} \zeta_j \right)  - \sqrt{\eta} \\
	&\geq \Pr_{\otimes_{\ell \in [k]} \nu_\ell} \left(\zeta_1 > \max_{j \in [k] \setminus \{1\}} \zeta_j\right) - k\sqrt{\eta}. 
	\end{align*}
\end{proof}

\subsection{Technical Lemmas for Almost Correlated Random Variables (at least $(1-\eta)$-correlation)}

\begin{lem}				\label{lem:corr-conc}
	Let $X,Y$ be $(1-\eta)$-correlated identically distributed random variables with $\E[X] = \E[Y] = 0$ and $\E[X^2] = \E[Y^2]  = 1$. Then 
	\[
	\Pr\left(|X - Y| \geq \eta^{1/4}\right) \leq 2\sqrt{\eta}
	\]
\end{lem}
\begin{proof}
	We begin by observing that due to the first and second moment constraints we have $\E[XY] = {\rm Corr}(X,Y) = 1 - \eta$. Then we can bound the second moment of the random variable $|X-Y|$ as
	\[
	\E\left[(X-Y)^2\right] = \E[X^2] + \E[Y^2] - 2\E[XY] \leq 2 - 2(1 - \eta) \leq 2 \eta.
	\]
	Hence using Markov's inequality we get that
	\[
	\Pr\left(|X - Y| \geq \alpha\right) \leq \Pr\left(|X - Y|^2 \geq \alpha^2\right) \leq \frac{\E\left[|X-Y|^2\right]}{\alpha^2} \leq \frac{2\eta}{\alpha^2}.
	\]
	Setting $\alpha = \eta^{1/4}$ in the above completes the proof.
\end{proof}

\section{The Seq-PB Algorithm~\cite{SG20}}				\label{app:seq-pb}

The Seq-PB Algorithm from \cite{SG20} is an $(\epsilon,\delta)$-PAC algorithm for the best arm determination problem under the \irum~choice model. Informally, the algorithm proceeds as follows: at every iteration $\ell$, the algorithm maintains a set of arms $S_\ell \subseteq [n]$ which acts as the candidate set for the best arms. Then at any iteration $\ell$, the algorithm considers a partition $\cG_{\ell,1} \uplus \cG_{\ell,2} \uplus \cdots \uplus \cG_{\ell,\lceil |S_\ell|/k \rceil}$ of $S_\ell$ into $k$-sized sets and plays each subset $t_\ell = O(k/\epsilon^2_\ell \log(n/\delta_\ell))$ times (where $\epsilon_\ell,\delta_\ell$ are geometrically decreasing as functions of $\ell$). Now given the feedback from the above subsetwise queries, the algorithm then proceeds to construct the next set of candidate winners $S_{\ell + 1} \subseteq S_\ell$ by retaining one item $i_{\ell,j}$ from each group $\cG_{\ell,j}$ -- in particular, the item $i_{\ell,j}$ is the item with the largest win count among the $t_\ell$-independent plays of the subset $\cG_{\ell,j}$. 

Overall, the sequence of parameters $(\epsilon_\ell,\delta_\ell)$ are set in a way such that they satisfy $\sum_{\ell} \epsilon_\ell \leq \epsilon$, $\sum_{\ell} \delta_\ell \leq \delta$ and in addition, the algorithm maintains the following iterative invariant: at any iteration $\ell$, the set $S_\ell$ retains at least one $\sum_{j\leq \ell} \epsilon_j$-best arm with probability at least $1 - \delta_\ell$. Furthermore, since at any iteration, the algorithm carries over only $1/k$-fraction of items for the next iteration, in $t^* := O(\log_k n)$-steps, the algorithm would converge to a singleton set $S_{t^*}$ which is guaranteed to have an $\epsilon$-best arm with probability at least $1 - \delta$. We refer interested readers to \cite{SG20} for more details on the Seq-PB algorithm.

\end{document}